\documentclass{article}
\usepackage{arxiv}

\usepackage{algorithm,algorithmic}
\usepackage{amsmath,amsthm,amssymb}
\usepackage{colortbl,soul,xcolor,bbm,caption,subcaption}
\usepackage{booktabs,multicol,multirow,makecell,longtable} 
\usepackage{mathtools}
\usepackage{xspace}
\usepackage{epsfig}
\usepackage{url,hyperref}

\DeclareMathOperator*{\argmax}{\arg\!\max}
\DeclareMathOperator*{\argmin}{\arg\!\min}

\def\mathcolor#1#{\@mathcolor{#1}}
\def\@mathcolor#1#2#3{%
  \protect\leavevmode
  \begingroup
    \color#1{#2}#3%
  \endgroup
}

\newcommand{\hlc}[2][yellow]{{%
    \colorlet{foo}{#1}%
    \sethlcolor{foo}\hl{#2}}%
}
\newcommand{\Mod}[1]{\ (\mathrm{mod}\ #1)}

\newtheorem{thm}{Theorem}

\newtheorem{lma}{Lemma}

\newtheorem{exm}{Example}

\newtheorem{cor}{Corollary}

\newtheorem{pro}{Proposition}

\newcommand{\ignore}[1]{}
\AtBeginDocument{%
  \providecommand\BibTeX{{%
    \normalfont B\kern-0.5em{\scshape i\kern-0.25em b}\kern-0.8em\TeX}}}

\begin{document}

\title{Analysis of Evolutionary Diversity Optimization for Permutation Problems}
\date{}

\author{Anh Viet Do
\\Optimisation and Logistics
\\The University of Adelaide, Adelaide, Australia
\And
Mingyu Guo
\\Optimisation and Logistics
\\The University of Adelaide, Adelaide, Australia
\And
Aneta Neumann
\\Optimisation and Logistics
\\The University of Adelaide, Adelaide, Australia
\And
Frank Neumann
\\Optimisation and Logistics
\\The University of Adelaide, Adelaide, Australia
}
\maketitle

\begin{abstract}
Generating diverse populations of high quality solutions has gained interest as a promising extension to the traditional optimization tasks. This work contributes to this line of research with an investigation on evolutionary diversity optimization for three of the most well-studied permutation problems, namely the Traveling Salesperson Problem (TSP), both symmetric and asymmetric variants, and Quadratic Assignment Problem (QAP). It includes an analysis of the worst-case performance of a simple mutation-only evolutionary algorithm with different mutation operators, using an established diversity measure. Theoretical results show many mutation operators for these problems guarantee convergence to maximally diverse populations of sufficiently small size within cubic to quartic expected run-time. On the other hand, the result on QAP suggests that strong mutations give poor worst-case performance, as mutation strength contributes exponentially to the expected run-time. Additionally, experiments are carried out on QAPLIB and synthetic instances in unconstrained and constrained settings, and reveal much more optimistic practical performances, while corroborating the theoretical finding regarding mutation strength. These results should serve as a baseline for future studies.
\end{abstract}


\maketitle

\section{Introduction}

\emph{Evolutionary diversity optimization} (EDO) aims to compute a set of diverse solutions that all have high quality while maximally differing from each other. This area of research started by Ulrich and Thiele~\cite{Ulrich2010,Ulrich2011}\footnote{The idea of finding maximally diverse solutions with genetic algorithms can be traced back to \cite{Ronald}} has recently gained significant attention within the evolutionary computation community, as evolution itself is increasingly regarded as a diversification device rather than a pure objective optimizer \cite{Pugh2016}. After all, in nature, deviating from the predecessors leads to finding new niches, which reduces competitive pressure and increases evolvability \cite{Lehman2013}. This perspective challenges the notion that evolutionary processes are mainly adaptive with respect to some quality metrics, and that population diversity is only in service of adapting its individuals and is without intrinsic worth. In applications, diversity optimization is a useful extension to the traditional optimization tasks, as a set of multiple interesting solutions has more practical value than a single very good solution.

Formally, given an objective function $f$ to be minimized (the maximization variant is defined similarly) over a feasible solution space $S$, a threshold value $F$, a natural number $\mu$, and a diversity function $div$ over solution sets, EDO searches for a solution set $P$ such that

\begin{equation}\label{eq:problem}
P\in\argmax_{Q\subseteq 2^S:|Q|=\mu}\left\lbrace div(Q)\mid\forall x \in Q,f(x)\leq F\right\rbrace.
\end{equation}

EDO research arose amid the interest in diverse solutions problems in the broader optimization research, which is currently relevant \cite{Ingmar2020,Baste2020,hanaka2020finding}. This class of problem addresses the practical necessity of having multiple solutions, such as providing alternatives that are, by being diverse, robust in allowing quick adaptation to changes in the problems. Furthermore, it helps the users be more flexible in adjusting for gaps between the problem models and real-world settings, arising frequently from modeling errors, and imprecise/uncertain aspects of the problem \cite{Schittekat2009}. Additionally, diverse solution sets contain rich information about the problem instance (as opposed to similar solution sets), which aids the users in making better decisions. While one could make use of existing solution enumeration techniques to obtain such values, the number of relevant solutions can grow rapidly \cite{Glover2000}, overburdening the decision makers. Moreover, top-k enumeration, while being expensive, often yields highly similar solutions \cite{Wang2013,Yuan2015}. These reasons justify separate treatments of diverse solutions problems from those of enumeration and multi-modal optimization.

\subsection{Related work}

Before diversity of solutions became of interest as it is in EDO, researchers in evolutionary computation considered the ``multi-solution problems'', in which many solutions of interest are sought. This has been one of the main motivations and applications of \emph{Evolutionary multi-modal optimization} \cite{Shir2012,wong2015evolutionary,Preuss2015,Li2017}. Most effort in this area has been spent on continuous search spaces, while fewer works such as \cite{Ronald,Angus2006,Han2018,Huang2018,Huang2019,Huang2020} deal with combinatorial problems. It is important to note that in multi-modal optimization, diversity of solutions primarily serves as a necessary property of local optima discovery processes, and not as an optimization objective in its own right. In fact, multi-solution in this context is mostly regarded as an intermediate problem, the solutions to which facilitate the search for global optima.

On the other hand, there have been studies in evolutionary computation that explore different relationships between quality and diversity. These include a trend emerging from the evolutionary robotics that is \emph{Quality Diversity}, which focuses on exploring diverse niches in the feature space and maximizing quality within each niche \cite{Pugh2016,Cully2018,Gravina2019,Alvarez2019}. This paradigm seeks to maximize diversity via niches discovery, relying on determining a well-defined notion of niches. Other studies place more importance on diversity measured directly from solutions, applying evolutionary techniques to generate images with varying features \cite{Alexander2017}, or to compute diverse Traveling Salesperson Problem (TSP) instances \cite{Gao2020,Bossek2019} useful for automated algorithm selection and configuration \cite{Kerschke2019}. 
Different indicators for measuring the diversity of sets of solutions in EDO algorithms such as the star discrepancy~\cite{Neumann2018}
or popular indicators from the area of evolutionary multi-objective optimization~\cite{Neumann2019} have been investigated to create high quality sets of solutions. The studies \cite{Nikfarjam2021,Nikfarjam20211} explore EDO for symmetric TSP solutions using entropy measure, while non-scalar diversity measures are proposed in \cite{Do2020}. Others also consider EDO in knapsack problems \cite{Bossek2021}, minimum spanning tree problem \cite{Bossek20211}, and submodular optimization \cite{Neumann2021}.

Outside the realm of evolutionary computation, the diverse solutions problem has been studied as an extension to many important classes of difficult problems. Many interesting results have been discovered for diverse solutions to constraint satisfaction and optimization problems \cite{Emmanuel05,Petit15,Ruffini2019}. Extensions to the traditional solving paradigms have been made to compute diverse solutions to SAT and Answer Set Problem, using existing powerful frameworks \cite{Nadel2011,Eiter2009}. Others have proposed methods for Mixed Integer Programming that incorporate diversity into quality-based heuristics \cite{Glover2000,Danna,Trapp2015}. More recently, the first provably fixed-parameter tractable algorithms have been proposed for diverse solutions to a broad class of graph-based vertex problems \cite{Baste2020}, via modification of dynamic programming on the graph's tree decomposition. This inspired subsequent research on other combinatorial structures such as trees, paths \cite{hanaka2020finding,Hanaka2022}, matching \cite{Fomin20211}, independent sets \cite{Fomin2021}, and linear orders \cite{Arrighi2021}. On the other hand, a general modeling framework has been proposed for diverse solutions to any combinatorial problem \cite{Ingmar2020}.

\subsection{Our contribution}

We contribute to the understanding of evolutionary diversity optimization on combinatorial problems, mainly from the theoretical run-time perspective. We refer to \cite{DBLP:books/daglib/0025643,BookDoeNeu} for comprehensive overviews of run-time analysis of discrete evolutionary optimizers. Specifically, we focus on symmetric and asymmetric TSP (abbreviated as STSP and ATSP, respectively), and Quadratic Assignment Problem (QAP), two classical NP-hard problems where solutions are represented as permutations, and the latter of which has also been attempted with genetic algorithms \cite{Tate1995,Tosun2014,Misevicius2004,Ahuja2000}. The structures of the solution spaces associated with these problems are similar, yet different enough to merit distinct diversity measures. We use two approaches to measuring diversity: one based on the representation frequencies of ``objects'' (edges or assignments) in the population, and one based on the minimum distance between each solution and the rest. We consider the simple evolutionary algorithm that only uses mutation, and examine its worst-case performances in diversity maximization when various mutation operators are used. Our results reveal how properties of a population influence the effectiveness of mutations in equalizing objects' representation frequencies. Additionally, we carry out experimental benchmark on various QAPLIB instances in unconstrained (no quality threshold) and constrained settings, using a simple mutation-only algorithm with 2-opt mutation. The results indicate optimistic run-time to maximize diversity on QAP solutions, and show maximization behaviors when using different diversity measures in the algorithm. These contributions are included in the conference version of this article, published in the proceedings of GECCO 2021 \cite{Do2021}. Note that the experimental results on TSP already presented in \cite{Do2020} are not included here. Nevertheless, those results also exhibit, when the population is small, lower run-time than our worst-case analysis implies.

This article extends the conference version by first expanding the scope of the worst-case run-time result in QAP: we generalize the proof to k-opt mutation from 2-opt mutation, to account for larger mutation strength choices (Section \ref{sec:qap}). Secondly, we augment the lemma for this result by proving constructively that the condition for the lack of local optima is non-trivial in case of 2-opt. Thirdly, we add theoretical results for the algorithm on ATSP using 3-opt and 4-opt mutation (Section \ref{sec:atsp}). Finally, we perform additional experimentation with unconstrained scenarios in QAP, using exhaustive combinations of problem size and population size values, and varying mutation strengths (Section \ref{sec:unconstrained}).

Our results regarding ATSP tours show that strict improvements are guaranteed at population sizes greater than the upper bound in the corresponding results regarding STSP tours. However, this extra flexibility comes at a cost of increased asymptotic worst-case expected run-time. We reveal similar insight in diversifying QAP solutions: the mutation strength contributes exponentially to the expected run-time, which is most significant at near-optimal diversity. This theoretical phenomenon from the use of strong mutations aligns with our experimental results. Furthermore, we observe in the experiment that the hard cases for the algorithm occur when the population size is close to multiples of the problem size, and that in all other cases, average run times to reach the optimum are significantly lower than the worst-case run-time.

The paper is structured as follows. In Section~\ref{sec:problem}, we introduce the STSP/ATSP and QAP in the context of evolutionary diversity optimization and describe the algorithm that is the subject of our analysis. In Section~\ref{sec:method}, we define the diversity measures used for the three problems. Section~\ref{sec:analysis} includes the run-time analysis of the algorithm. We report on our experimental investigations in Section~\ref{sec:experiment} and finish with some conclusions.

\section{Maximizing diversity in STSP, ATSP and QAP}\label{sec:problem}

Throughout the paper, we use the shorthand $[n]=\{1,\ldots,n\}$. The symmetric STSP is formulated as follow. Given a complete undirected graph $G=(V,E)$ with $n=|V|$ nodes, $m=n(n-1)/2=|E|$ edges and the distance function $d:V\times V\to\mathbb{R}_{\geq 0}$, the goal is to compute a tour of minimal cost that visits each node exactly once and finally returns to the original node. Let $V=[n]$, the goal is to find a tour represented by the permutation $\pi:V\to V$ that minimizes the tour cost
\begin{align*}
c(\pi) = d(\pi(n),\pi(1)) + \sum_{i=1}^{n-1} d(\pi(i),\pi(i+1)). 
\end{align*}

An ATSP instance is defined with a directed graph containing $n(n-1)$ edges, an asymmetric edge weight function $d$, and an identical cost function. The QAP is formulated as follow. Given facilities $F=\{f_1,\dots,f_n\}$, locations $L=\{l_1,\dots,l_n\}$, weights $w:F\times F\to\mathbb{R}_{\geq 0}$, flows $f:L\times L\to\mathbb{R}_{\geq 0}$, find a 1-1 mapping $a:F\to L$ that minimizes the cost function
\[c(a)=\sum_{i,j\in F}w(i,j)f(a(i),a(j)).\]
A problem instance is encoded with two $n\times n$ matrices: one for $w$ and one for $f$. Similar to STSP and ATSP, we can abstract $F$ and $L$ like we do $V$: $F=[n]$ and $L=[n]$. Therefore, each mapping is uniquely defined by a $[n]\to[n]$ permutation. Given that there is a 1-to-1 correspondence between all permutations and all mappings, the solution space is the permutation space. This is an important distinction between STSP/ATSP and QAP from which low-level differences between the diversity measures in each case emerge. Furthermore, a STSP tour corresponds to 2 edge-disjoint ATSP tours, which highlights the differences between the two types of search spaces. On the other hand, the high level structure of a tour (directed or undirected) is identical to that of a mapping, so the notions like distance or diversity are the same for all three problems above a certain layer of abstraction.

In this paper, we consider diversity optimization for STSP, ATSP and QAP, which is a special case of \eqref{eq:problem}. For each problem instance, we are to find a set $P$ of $\mu = |P|$ solutions that is diverse with respect to some diversity measure, while each solution meets a given quality threshold. Typically, this threshold is set to be $(1+\alpha)OPT$, where $OPT$ denotes the optimal objective value and $\alpha>0$ decides the optimality gap. Such a formulation requires that the final population only contains $(1+\alpha)$-approximations for a problem instance. We assume that the optimal solution is known for a given instance, which does not eliminate the problem's intractability. We refer to such an instance a $(\mu,\alpha)$-instance of the diversity optimization problem.

We consider $(\mu+1)$-EA which was used to diversify STSP tours \cite{Do2020}. The algorithm is described in Algorithm \ref{alg:ea}; it takes in the raw threshold value instead of $\alpha$, and the diversity measure to be maximized. It uses only mutation to introduce new genes, and tries to minimize duplication in the gene pool with elitist survival selection. The algorithm slightly modifies the population in each step by mutating a random solution, essentially performing random local search in the population space. As with many evolutionary algorithms, it can be customized for different problems, in this case by modifying the mutation operator and the diversity measure. In this work, we are interested in worst-case performances of the algorithm under the assumption that any offspring is acceptable. We consider the usual black-box complexity model, where the run-time is defined as the number of fitness evaluations \cite{BookDoeNeu}. For $(\mu+1)$-EA, it is the same as the number of iterations.

\begin{algorithm}[ht]
\begin{algorithmic}[1]
\STATE \textbf{Inputs:} instance $c$, set size $\mu$, threshold value $F$, diversity measure function $div$
\STATE $P\gets$ initial population
\WHILE{stopping criteria not met}
\itemindent=0pc
\STATE $I\gets randomSelect(P)$
\STATE $I'\gets mutate(I)$
\IF{$c(I')\leq F$}
\itemindent=0pc
\STATE $P\gets P\cup\{I'\}$
\STATE $I''\gets\argmin_{J\in P}\{div(P\setminus\{J\})\}$
\STATE $P\gets P\setminus\{I''\}$
\ENDIF
\ENDWHILE
\STATE \textbf{return} $P$
\end{algorithmic}
\caption{$(\mu+1)$-EA for diversity optimization}
\label{alg:ea}
\end{algorithm}

\section{Diversity measures}\label{sec:method}

The structure of a STSP/ATSP tour is similar to that of a QAP mapping in the sense that they are both each defined by a set of objects: edges in tours and assignments in mappings. In fact, the size of such a set is always equal to the instance size $n$. For this reason, diversity measures for populations of tours, and those for populations of mappings share many commonalities. In particular, we describe two measures introduced in \cite{Do2020}, customized for STSP, ATSP and QAP. For consistency, we use the same notations for the same concepts between the three problems unless told otherwise. We also refer to \cite{Do2020} for more in-depth discussion on the measures, and fast implementations of the survival selection for Algorithm \ref{alg:ea} based on these measures, which can be customized for ATSP and QAP solutions. We remark that while QAP is a generalization of STSP and ATSP, their corresponding diversity optimization problems do not share this relationship.

\subsection{Edge/Assignment diversity}

In this approach, we consider diversity in terms of equal representations of edges/assignments in the population. It takes into account, for each object, the number of solutions containing it, among the $\mu$ solutions in the population.

For STSP and ATSP, given a population of tours $P$ and an edge $e \in E$, we denote by $n(e, P)$ its edge count, which is defined,
\begin{align*}
    n(e,P)=\left|\{T\in P\mid e\in E(T)\}\right|\in\{0,\ldots,\mu\}
\end{align*}
where $E(T) \subset E$ is the set of edges used by tour $T$. Then in order to maximize the edge diversity we aim to minimize, in the lexicographic order, the vector
\begin{align}\label{eq:N_div_TSP}
    \mathcal{N}(P)=\text{sort}\left(n(e_1, P),n(e_2, P),\ldots,n(e_m, P)\right),
\end{align}
where sorting is performed in descending order. As shown in \cite{Do2020}, since the total edge count is fixed, this equalizes the counts across edges, thus maximizing the pairwise distances sum
\begin{align*}
    D_1(P)=\sum_{T_1\in P}\sum_{T_2\in P}|E(T_1)\setminus E(T_2)|=|P|(|P|-1)n+\sum_{e\in E}n(e,P)(1-n(e,P)).
\end{align*}

Similarly for QAP, given a population of mappings $P$, we denote by $n(i,j,P)$ its assignment count as follow,
\begin{align*}
    n(i,j,P)=\left|\{a \in P\mid(i,j)\in A(a)\}\right|\in\{0, \ldots, \mu\}
\end{align*}
where $A(a) \subset [n]\times[n]$ is the set of assignments used by solution $a$. The corresponding vector to be minimized in order to maximize assignment diversity is then
\begin{align}\label{eq:N_div_QAP}
    \mathcal{N}(P) = \text{sort}\left(n(i,j, P)\right)_{i,j\in[n]},
\end{align}
in the lexicographic order where sorting is performed in descending order. Similarly, this maximizes the following quantity
\begin{align*}
    D_1(P) = \sum_{a\in P}\sum_{b\in P}|A(a)\setminus A(b)|.
\end{align*}

While this diversity measure is directly related to the notion of diversity, using it to optimize populations has its drawbacks. As mentioned in \cite{Do2020}, populations containing clustering subsets of solutions can have high $D_1$ score, which is undesirable. For this reason, we also consider another measure that circumvents this issue.

\subsection{Equalizing pairwise distances}
Instead of maximizing all pairwise distances at once, this approach focuses on maximizing smallest distances, potentially reducing larger distances as a result. Optimizing for this measure reduces clustering phenomena, as well as tends to increase the distance sum. In this approach, we minimize the following vector lexicographically
\begin{align}\label{eq:D_div}
\mathcal{D}(P) = \text{sort}\left(\left(o_{X,Y}\right)_{X,Y\in P}\right),
\end{align}
where sorting is performed in descending order, and $o_{XY}=|E(X)\cap E(Y)|$ if $X$ and $Y$ are STSP tours, and $o_{XY}=|A(X)\cap A(Y)|$ if they are QAP mappings. Doing this would also maximize the following quantity
\begin{align*}
D_2(P)=\sum_{T\in P}\min_{X\in P\setminus\{T\}}\left\lbrace|E(T)\setminus E(X)|\right\rbrace,\quad\text{or}\quad D_2(P)=\sum_{a\in P}\min_{b\in P\setminus\{T\}}\left\lbrace|A(a)\setminus A(b)|\right\rbrace.
\end{align*}
We know from Hamiltonian decomposition of complete undirected graphs (Theorem 1 in \cite{Alspach1990}) that for any STSP tour population $P$ of size at most $\left\lfloor\frac{n-1}{2}\right\rfloor$, we have
\begin{align*}
\argmin_P\{\mathcal{N}(P)\}=\argmin_P\{\mathcal{D}(P)\}=\argmin_P\{D_1(P)\}=\argmin_P\{D_2(P)\}.
\end{align*}
One of the results in this study implies that the same is true for any QAP mapping population of size at most $n$. On the other hand, when $\mu>n$, $P^*\in\argmax_P\{D_2(P)\}$ doesn't necessarily imply $P^*\in\argmin_P\{\mathcal{D}(P)\}$, as shown by the following example.
\begin{exm}
For a QAP instance where $n=4$ and $\mu=5$, let $a_1=(1,2,3,4)$, $a_2=(1,3,4,2)$, $a_3=(3,2,4,1)$, $a_4=(2,4,3,1)$, $a_5=(2,3,1,4)$, $a_6=(4,2,1,3)$, $a_7=(3,1,2,4)$, $P=\{a_1,a_2,a_3,a_4,a_5\}$, $P'=\{a_1,a_2,a_4,a_6,a_7\}$, we have $D_2(P)=D_2(P')=15$ which is the maximum. However,
\begin{align*}
\mathcal{D}(P)=(1,1,1,1,1,1,1,1,0,0)>
\mathcal{D}(P')=(1,1,1,1,0,0,0,0,0,0).
\end{align*}
\end{exm}
Because of this, it is tricky to determine the maximum achievable diversity $\mathcal{D}$ in such cases. For now, we assume the upper bound $\mu n$ of $D_2$, which is relevant to our experimentation in Section \ref{sec:experiment}.

\section{Properties and worst-case results}\label{sec:analysis}

We investigate the theoretical performance of Algorithm \ref{alg:ea} in optimizing for $\mathcal{N}$ defined in \eqref{eq:N_div_TSP} and \eqref{eq:N_div_QAP} without the quality criterion. For STSP, we consider three mutation operators: 2-opt, 3-opt (insertion) and 4-opt (exchange) on the visit-order representation. For ATSP, we consider 3-opt and 4-opt, the former of which is typically used in local search heuristics \cite{Kanellakis1980,Cirasella2001}. For QAP, we consider the k-opt mutation which is a generalization of the 2-opt transposition. Regarding the solution's representation, we assume that each of the operators samples its corresponding neighborhood (in the underlying space and not the representation space) uniformly, so the choice of representation is only a matter of implementation. As mentioned, we use, in our proofs, the visit-order representation for STSP/ATSP solutions, and the natural assignment permutation for QAP solutions. As we will see in the proofs, these representations can be used to index the subset of the neighborhood that is relevant to the algorithm's behaviors, thus allowing for accurate counting.

For the analysis, we are interested in the number of iterations until a population with optimal diversity is achieved. Our derivation of results is predicated on the lack of local optima: we only consider scenarios where it is always possible to strictly improves diversity in a single step of the algorithm. Formally, we say that Algorithm \ref{alg:ea}, using the mutation operator sampling from the neighborhood $B(\cdot)$, encounters a local optimum if its current population, $P$, is such that
\[\forall a,b\in P,\forall b'\in B(b),\mathcal{N}(P)\leq\mathcal{N}(P\setminus\{a\}\cup\{b'\}).\]

This consideration facilitates positive lower bounds of the success rate, allowing us to apply the drift analysis technique \cite{Lengler2019}. Since this technique is intuitive, we do not explicitly mention it in the proofs for brevity's sake.

\subsection{STSP}\label{sec:tsp}

Let $d_P=\max_{e\in E}\{n(e,P)\}$ and $c_P=|{e\in E\mid n(e,P)=d_P}|$. For each node $i$, let $in(i)$ be the set of edges incident to $i$, and $s(i,P)=\sum_{e\in in(i)}n(e,P)$. For each tour $I$, let $2opt(I,i,j)$ be the tour resulted from applying 2-opt to $I$ at positions $i$ and $j$ in the permutation, and $4opt(I,i,j)$ be the tour from exchanging $i$-th and $j$-th elements in $I$. We assume $n\geq4$ as the other cases are trivial. Note that we must have $n\geq6$ for 4-opt to be applicable, so it is implicitly assumed when appropriate.

First, we show that any population with sub-optimal diversity and of sufficiently small size presents no local optima to the Algorithm \ref{alg:ea} with 2-opt mutation, while deriving a lower bound of the probability where a strict improvement is made in a single step. Here, we regard a single-step improvement as the reduction of either $c_P$ or $d_P$, as aligned with the algorithm's convergence path. Furthermore, it has been shown, using Hamiltonian cycle decomposition, that for any $\mu\leq\left\lfloor\frac{n-1}{2}\right\rfloor$, there is a $\mu$-size population $P$ where $d_P=1$ \cite{Do2020}. As such, within this context, any sub-optimally diverse population $P$ has $d_P\geq2$.

\begin{lma}\label{lemma:ea_step_2opt_tsp}
Given a population of tours $P$ such that $2\leq\mu\leq\left\lfloor\frac{n+2}{4}\right\rfloor$ and $d_P\geq2$, there exists a tour $I\in P$ and a pair $(i,j)$, such that $P'=(P\setminus\{I\})\cup\{2opt(I,i,j)\}$ satisfies,
\begin{equation}\label{eq:improvement}
(c_P>c_{P'}\wedge d_P=d_{P'})\vee d_P>d_{P'}.
\end{equation}
Moreover, in each iteration, the Algorithm \ref{alg:ea} with 2-opt mutation on a $(\mu,\infty)$-instance makes such an improvement with probability at least \[\frac{2[(n-1)(d_P-2)+1]}{\mu n(n-3)}.\]
\end{lma}
\begin{proof}
There must be $d_P$ tours $I$ in $P$ containing edge $e$ such that $,n(e,P)=d_P$, let $I$ be one such tour. W.l.o.g, let $I$ be represented by a permutation of nodes $(i_1,i_2,\ldots,i_n)$ where $n(\{i_1,i_2\},P)=d_P$. The operation $2opt(I,2,k)$ trades edges $\{i_1,i_2\}$ and $\{i_k,i_{k+1}\}$ in $I$ for $\{i_1,i_k\}$ and $\{i_2,i_{k+1}\}$. If for every such new edge $e'$ we have $n(e',P)<d_P-1$, then  $P'=(P\setminus\{I\})\cup\{2opt(I,2,k)\}$ satisfies \eqref{eq:improvement} since $n(\{i_1,i_k\},P')$ and $n(\{i_2,i_{k+1}\},P')$ cannot reach $d_P$. We show that there is always such a position $k$. Since $k$ can only go from $3$ to $n-1$, there are $n-3$ choices of $k$. It's the case that $s(i,P)=2\mu$ for any $i$ since each tour contributes $2$ to $s(i,P)$, and that $n(\{i_n,i_1\},P)\geq1$ and $n(\{i_2,i_3\},P)\geq1$ since $I$ contains them, thus
\begin{align}\label{eq:remaining_counts_2opt_tsp}
\sum_{k=3}^{n-1}n(\{i_1,i_k\},P)\leq2\mu-d_P-1,\quad\text{and }\quad\sum_{k=4}^{n}n(\{i_2,i_k\},P)\leq2\mu-d_P-1.
\end{align}
According to the Pigeonhole Principle, \eqref{eq:remaining_counts_2opt_tsp} implies there are at least $\delta$ elements $k$ from $3$ to $n-1$ such that $n(\{i_1,i_k\},P)<d_P-1$, where
\[\delta=n-3-\left\lfloor\frac{2\mu-d_P-1}{d_P-1}\right\rfloor.\]
Likewise, there are at least $\delta$ elements $k$ from $4$ to $n$ such that $n(\{i_2,i_k\},P)<d_P-1$. This implies that there are at least $\sigma$ elements $k$ from $3$ to $n-1$ such that $n(\{i_1,i_k\},P)<d_P-1$ and $n(\{i_2,i_{k+1}\},P)<d_P-1$, where
\begin{equation}\label{eq:sigma_2opt_tsp}
\sigma=2\delta-n+3=n-3-2\left\lfloor\frac{2\mu-d_P-1}{d_P-1}\right\rfloor.
\end{equation}
We can see that $\sigma\geq1$ when
\begin{equation*}\label{eq:mu_bound_2opt_tsp}
\mu\leq\left\lfloor\frac{(n-3)(d_P-1)+2d_P+1}{4}\right\rfloor.
\end{equation*}
This proves the first part of the lemma since $d_P\geq2$. In each iteration, the Algorithm \ref{alg:ea} selects a tour like $I$ with probability at least $d_P/\mu$. There are at least $\sigma$ different 2-opt neighbors on such a tour to produce $P'$. Since there are $n(n-3)/2$ 2-opt neighbors in total, the probability that the Algorithm \ref{alg:ea} obtains $P'$ from $P$ is at least
\[\frac{d_P}{\mu}\frac{2\sigma}{n(n-3)}\geq\frac{(n-1)(d_P-2)+1}{d_P-1}\frac{2d_P}{\mu n(n-3)}\geq\frac{2[(n-1)(d_P-2)+1]}{\mu n(n-3)},\]
where the first inequality follows from \eqref{eq:sigma_2opt_tsp} and the upper bound of $\mu$.
\end{proof}

In Lemma \ref{lemma:ea_step_2opt_tsp}, only one favorable scenario is accounted for where both edges to be traded in have counts less than $d_P-1$. However, there are other situations where strict improvements would be made as well, such as when both swapped-out edges have count $d_P$. Furthermore, a tour to be mutated might contain more than 2 edges with such count, increasing the number of beneficial choices dramatically. Consequently, the derived probability bound is pessimistic, and the average success rate might be much higher. It also means that the bound of the range of $\mu$ is pessimistic and the lack of local optima is very probable at larger population sizes, albeit with reduced diversity improvement probability.

Intuitively, larger population sizes present more complex search spaces where local search approaches are more prone to reaching sub-optimal results. It is reasonable to infer that small population sizes make diversity maximization easier for Algorithm \ref{alg:ea}. However, for 3-opt mutation, local optima can still exist even with population size being as small as 3. Here, we show a simple construction of supposedly easy cases where 3-opt fails to produce any strict improvement.

\begin{exm}
For any STSP instance of size $n\geq8$ where $n$ is a multiple of 4, we can always construct a population of 3 tours having sub-optimal diversity, such that no single 3-opt operation on any tour can improve diversity. Let the first tour be $I_1=(i_1,i_2,\dots,i_n)$, we derive the second tour $I_2$ sharing only 2 edges with $I_1$ and containing edges that form a ``crisscrossing'' pattern on $I_1$,
\begin{align*}
I_2=&(i_1,i_{n-1},\dots,i_{2k+1},i_{n-2k-1},\dots,i_{n/2-1},i_{n/2+1},i_{n/2},i_{n/2+2},\dots,i_{n/2-2k},i_{n/2+2k},\dots,i_2,i_n).
\end{align*}
The third tour $I_3$ shares no edge with $I_1$ or $I_2$ and contains many edges that ``skip one node'' on $I_1$.
\begin{align*}
I_3=&(i_1,\dots,i_{2k+1},\dots,i_{n/2-1},i_{n/2+2},\dots,i_{n/2+2k},\dots,i_n,i_{n/2},\dots,i_{n/2-2k},\dots,i_2,i_{n-1},\dots,i_{n-2k-1},i_{n/2+1}).
\end{align*}
In order to improve diversity, the operation must exchange, on either tour, at least one edge with count 2. However, any 3-opt operation with such restriction ends up trading in at least another edge used by the other tours, nullifying any improvement it makes. This population presents a local optimum for algorithms that uses 3-opt as the only solution generating mechanism. Figure \ref{fig:3opt_example} illustrates two examples of the construction with $n=8$ and $n=12$.

\begin{figure}[t]
\centering
\includegraphics[width=.5\linewidth]{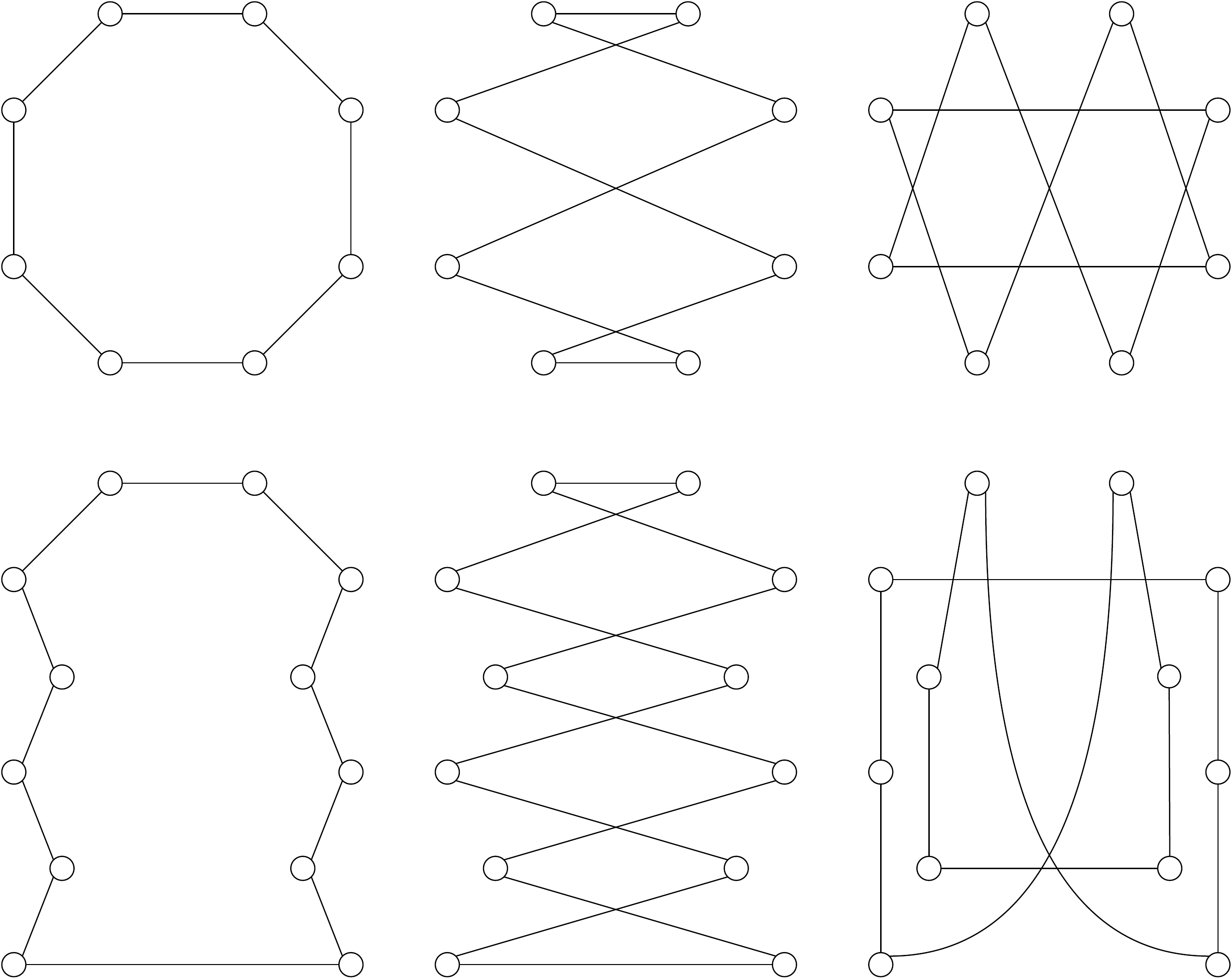}
\caption{
Examples of constructed tours with $n=8$ and $n=12$ where no single 3-opt operation on any tour improves diversity among tours in each row.}
\label{fig:3opt_example}
\end{figure}
\end{exm}

We speculate that in many cases, the insertion 3-opt suffers from its asymmetrical nature. Both 2-opt and 3-opt neighbors are each defined by two decisions. For 2-opt, the two decisions are which two edges to be exchanged, and only after both are made will the two new edges be fixed. For 3-opt, one decision determines which set of two adjacent edges to exchanged, and the other defines the third edge. Unlike 2-opt, after only one decision, one out of the three new edges is already fixed. Such limited flexibility makes it difficult to guarantee diversity improvements via 3-opt without additional assumptions about the population. In contrast, 4-opt is not subjected to this drawback, as the two decisions associated with it are symmetric. For this reason, we can derive another result for 4-opt similar to Lemma \ref{lemma:ea_step_2opt_tsp}.

\begin{lma}\label{lemma:ea_step_4opt_tsp}
Given a population of tours $P$ such that $2\leq\mu\leq\left\lfloor\frac{n+4}{8}\right\rfloor$ and $d_P\geq2$, there exists a tour $I\in P$ and a pair $(i,j)$, such that $P'=(P\setminus\{I\})\cup\{4opt(I,i,j)\}$ satisfies \eqref{eq:improvement}. Moreover, in each iteration, the Algorithm \ref{alg:ea} with 4-opt mutation on a $(\mu,\infty)$-instance makes such an improvement with probability at least \[\frac{4[(n-2)(d_P-2)+1]}{\mu n(n-5)}.\]
\end{lma}
\begin{proof}
There must be $d_P$ tours $I$ in $P$ containing edge $e$ such that $n(e,P)=d_P$, let $I$ be one such tour. W.l.o.g, let $I$ be represented by a permutation of nodes $(i_1,i_2,\ldots,i_n)$ where $n(\{i_1,i_2\},P)=d_P$. The operation $4opt(I,2,k)$ trades edges $\{i_1,i_2\}$, $\{i_2,i_3\}$, $\{i_{k-1},i_k\}$, $\{i_k,i_{k+1}\}$ in $I$ for $\{i_1,i_k\}$, $\{i_3,i_k\}$, $\{i_2,i_{k-1}\}$, $\{i_2,i_{k+1}\}$. If for every such new edge $e'$ we have $n(e',P)<d_P-1$, then  $P'=(P\setminus\{I\})\cup\{4opt(I,2,k)\}$ satisfies \eqref{eq:improvement} following similar reasoning in the proof of Lemma \ref{lemma:ea_step_2opt_tsp}. We show that there is always such a position $k$. Since $k$ can only go from $5$ to $n-1$, there are $n-5$ choices of $k$. We use the fact that $s(i,P)=2\mu$ for any $i$, and that $n(\{i_2,i_3\},P)\geq1$ since $I$ uses them, thus
\begin{align}\label{eq:remaining_counts_4opt_tsp}
\sum_{k=5}^{n-1}n(\{i_2,i_{k-1}\},P)\leq2\mu-d_P-1,\quad\text{and }\quad\sum_{k=5}^{n-1}n(\{i_2,i_{k+1}\},P)\leq2\mu-d_P-1.
\end{align}
According to the Pigeonhole Principle, \eqref{eq:remaining_counts_4opt_tsp} implies there are at least $\delta$ elements $k$ from $5$ to $n-1$ such that $n(\{i_2,i_{k-1}\},P)<d_P-1$, where
\[\delta=n-5-\left\lfloor\frac{2\mu-d_P-1}{d_P-1}\right\rfloor.\]
Likewise, there are at least $\delta$ elements $k$ from $5$ to $n-1$ such that $n(\{i_2,i_{k+1}\},P)<d_P-1$. This implies that there are at least $2\delta-n+5$ elements $k$ from $5$ to $n-1$ where $n(\{i_2,i_{k-1}\},P)<d_P-1$ and $n(\{i_2,i_{k+1}\},P)<d_P-1$, which we will call condition 1. We denote the number by $\Delta$
\[\Delta=2\delta-n+5=n-5-2\left\lfloor\frac{2\mu-d_P-1}{d_P-1}\right\rfloor,\]
Using $n(\{i_1,i_n\},P)\geq1$, we similarly derive that there are at least $\delta$ element $k$ from $5$ to $n-1$ where $n(\{i_1,i_k\},P)<d_P-1$. However, we only have $n(\{i_3,i_4\},P)\geq1$, meaning there are at least $\delta'$ element $k$ from $5$ to $n-1$ such that $n(\{i_3,i_k\},P)<d_P-1$ where
\[\delta'=n-5-\left\lfloor\frac{2\mu-1}{d_P-1}\right\rfloor.\]
From this, we have that there are at least $\delta+\delta'-n+5$ elements $k$ from $5$ to $n-1$ where $n(\{i_1,i_k\},P)<d_P-1$ and $n(\{i_3,i_k\},P)<d_P-1$, which we will call condition 2. We denote the number by $\Delta'$
\[\Delta'=\delta+\delta'-n+5=n-5-\left\lfloor\frac{2\mu-d_P-1}{d_P-1}\right\rfloor-\left\lfloor\frac{2\mu-1}{d_P-1}\right\rfloor,\]
Finally, we can infer that there are at least $\sigma$ choices of $k$ such that both condition 1 and condition 2 are met, where
\begin{equation}\label{eq:sigma_4opt_tsp}
\sigma=\Delta+\Delta'-n+5=n-5-3\left\lfloor\frac{2\mu-d_P-1}{d_P-1}\right\rfloor-\left\lfloor\frac{2\mu-1}{d_P-1}\right\rfloor.
\end{equation}
We can see that $\sigma\geq1$ when
\begin{equation*}\label{eq:mu_bound_4opt_tsp}
\mu\leq\left\lfloor\frac{(n-5)(d_P-1)+3d_P+3}{8}\right\rfloor.
\end{equation*}
This proves the first part of the lemma since $d_P\geq2$. By symmetry, there are at least $\sigma$ choices of $k$ from $4$ to $n-2$ such that $P'=(P\setminus\{I\})\cup\{4opt(I,1,k)\}$ satisfies \eqref{eq:improvement}, meaning there are at least $2\sigma$ 4-opt neighbors of $I$ leading to such an improvement. In each iteration, the Algorithm \ref{alg:ea} selects a tour like $I$ with probability at least $d_P/\mu$. Since there are $n(n-5)/2$ 4-opt neighbors in total, the probability that the Algorithm \ref{alg:ea} obtains $P'$ from $P$ is at least
\[\frac{d_P}{\mu}\frac{4\sigma}{n(n-5)}\geq\frac{(n-2)(d_P-2)+1}{d_P-1}\frac{4d_P}{\mu n(n-5)}\geq\frac{4[(n-2)(d_P-2)+1]}{\mu n(n-5)},\]
where the first inequality follows from \eqref{eq:sigma_4opt_tsp} and the upper bound of $\mu$.
\end{proof}

Like in Lemma \ref{lemma:ea_step_2opt_tsp}, only one out of many favorable scenarios is considered in Lemma \ref{lemma:ea_step_4opt_tsp}, so the lower bound is strict. The range of the population size is smaller to account for the fact that the condition for such a scenario is stronger than the one in Lemma \ref{lemma:ea_step_2opt_tsp}. With these results, we derive run-time results for 2-opt and 4-opt, relying on the longest possible path from zero diversity to the optimum.

\begin{thm}\label{theo:runtime_tsp}
On a $(\mu,\infty)$-instance based on any STSP instance with $n\geq6$ nodes, and $\mu\geq2$, the Algorithm \ref{alg:ea} obtains a $\mu$-population with maximum diversity within expected time $\mathcal{O}(\mu^2n^3)$ if
\begin{itemize}
\item it uses 2-opt mutation and $\mu\leq\left\lfloor\frac{n+2}{4}\right\rfloor$,
\item it uses 4-opt mutation and $\mu\leq\left\lfloor\frac{n+4}{8}\right\rfloor$.
\end{itemize}
\end{thm}
\begin{proof}
In the worst case, the algorithm begins with $d_P=\mu$ and $c_P=n$. At any time, we have $c_P\leq\mu n/d_P$. Moreover, in the worst case, each improvement either reduces $c_P$ by $1$, or reduces $d_P$ by $1$ and sets $c_P$ to its maximum value. With $2\leq\mu\leq\left\lfloor\frac{n-1}{2}\right\rfloor$, the maximum diversity is achieved iff $d_P=1$ as shown in \cite{Do2020}. According to Lemma \ref{lemma:ea_step_2opt_tsp}, the expected run time Algorithm \ref{alg:ea} requires to reach maximum diversity when using 2-opt mutation is at most
\[\sum_{j=2}^{\mu}\frac{\mu n}{j}\frac{\mu n(n-3)}{2[(n-1)(j-2)+1]}=\mathcal{O}(\mu^2n^3).\]
On the other hand, Lemma \ref{lemma:ea_step_4opt_tsp} implies that when $2\leq\mu\leq\left\lfloor\frac{n+4}{8}\right\rfloor$, Algorithm \ref{alg:ea} with 4-opt mutation needs at most the following expected run time
\[\sum_{j=2}^{\mu}\frac{\mu n}{j}\frac{\mu n(n-5)}{2[(n-2)(j-2)+1]}=\mathcal{O}(\mu^2n^3).\]
\end{proof}

As expected, the simple algorithm requires only quadratic expected run-time to achieve optimal diversity from any starting population of sufficiently small size. The quadratic scaling with $\mu$ comes from two factors. One is the fact that Algorithm \ref{alg:ea} needs to select the ``correct'' tour to mutate out of $\mu$ tours. The other is the fact that up to $\mu-1$ tours need to be modified to achieve the optimum, and only one is modified in each step. The cubic scaling with $n$ comes from the quadratic number of possible mutation operations, and the number of edges to modify in each tour. Additionally, most of the run-time is spent on the ``last stretch'' when reducing $d_P$ from 2 to 1, as the rest only takes up $\mathcal{O}(\mu^2n^2)$ expected number of steps.

\subsection{ATSP}\label{sec:atsp}

For ATSP, it is clear that no two distinct directed tours are less than 3 edges apart. In this case, we consider 3-opt and 4-opt, examples of which are illustrated in Figure \ref{fig:undirected_opt_example}. Note that for any set of 3 or 4 edges removed, there is only one way to reconnect the segments, and that these segments can be empty (i.e. having only one node). Let $3opt(I,i,j,k)$ and $4opt(I,i,j,k,h)$ denote a 3-opt neighbor and 4-opt neighbor of $I$, respectively, where the parameters are the ending positions of the segments in ascending order. Figure \ref{fig:undirected_opt_example} then illustrates $3opt(I,1,3,6)$ and $4opt(I,1,3,5,7)$, where $I=(1,2,3,4,5,6,7,8)$. Once the parameters are fixed, the same edge exchange occurs regardless of which two segments are swapped in the permutation. We can see that for any directed tour containing $n$ edges, there are $\binom{n}{3}$ distinct 3-opt neighbors, and $\binom{n}{4}$ distinct 4-opt neighbors; these are also implied in \cite{Kanellakis1980}. Naturally, we assume $n\geq3$ for 3-opt, and $n\geq4$ for 4-opt. Lastly, here we use the same definitions of $d_P$ and $c_P$ in Section \ref{sec:tsp}, only with directed edges instead.

\begin{figure}[t]
\centering
\includegraphics[width=.65\linewidth]{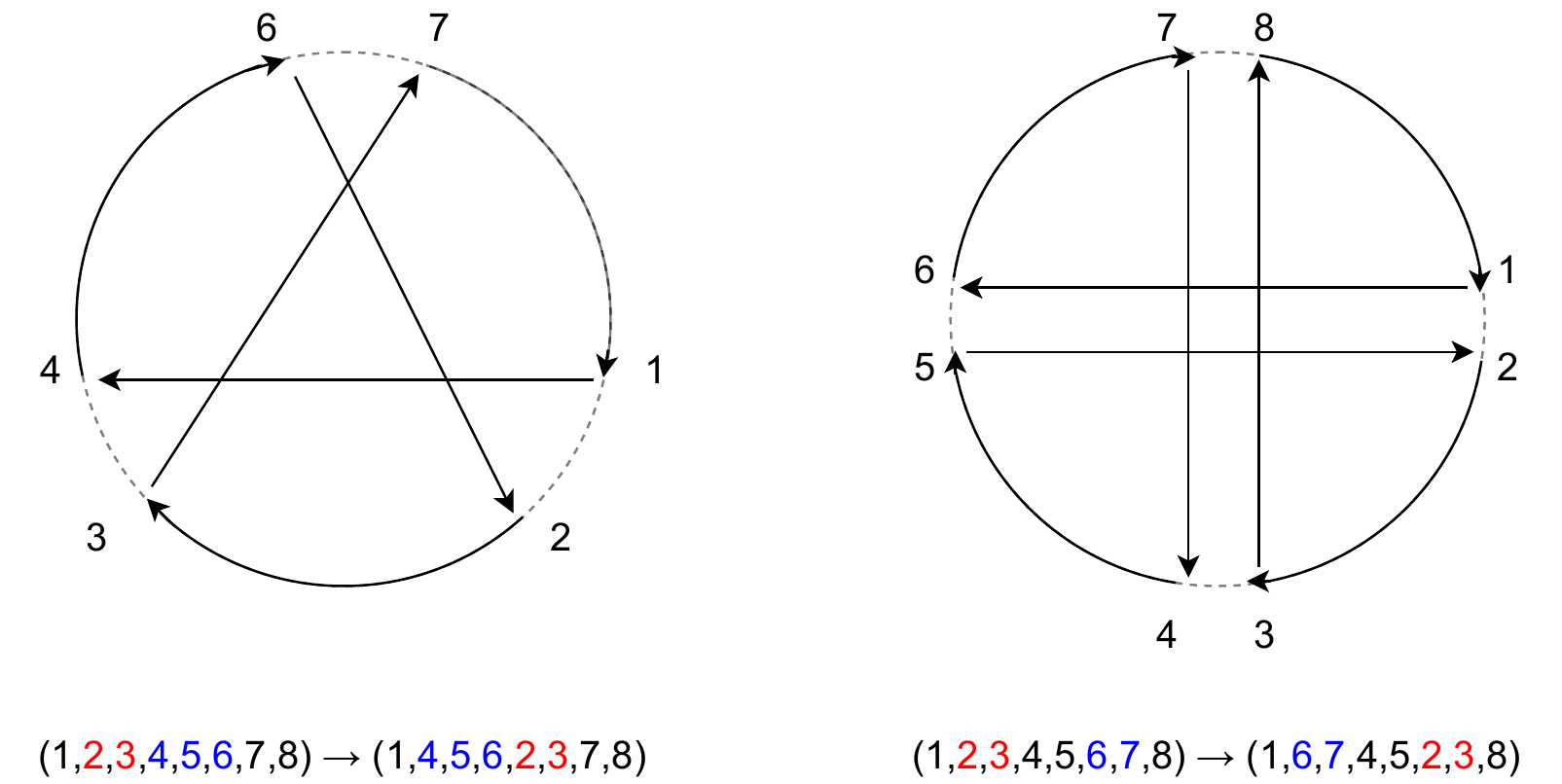}
\caption{
Examples of 3-opt and 4-opt, and their corresponding actions on the permutation.}
\label{fig:undirected_opt_example}
\end{figure}

We first prove that in ATSP, the achievable maximum diversity takes a similar form as in STSP, by relying on the corresponding proof in \cite{Do2020}.

\begin{cor}\label{theo:atsp_max_div}
Given $n\geq3$ and $\mu\geq1$, there exists a $\mu$-size population $P$ of tours in a complete directed graph $G=(V,E)$ where $|V|=n$ such that
\begin{equation}\label{eq:atsp_max_div}
\max_{e\in E}n(e,P)-\min_{e\in E}n(e,P)\leq1.
\end{equation}
\end{cor}
\begin{proof}
We prove by construction. First, we construct a population $P'$ of $\left\lceil\frac{\mu}{2}\right\rceil$ tours in a complete undirected graph $G'=(V,E')$ as specified in the proof of Theorem 1 in \cite{Do2020}, while keeping track of which tours are in $L$. Then, for each undirected tour $I\in P'$, we add two corresponding directed tours to $P$, obtained by imposing directions on $I$. Finally, if $\mu\equiv1\Mod{2}$, then remove a tour from $P$ corresponding to the last added tour into $P'$

For any $h=\{i,j\}\in E'$, we have $n(e,P')=n((i,j),P)=n((j,i),P)$ by the second step of the construction. Furthermore, if $h$ is in the last added tour in $P'$, then $n(e,P)=\max_{e\in E'}n(e,P')$ according to the proof in \cite{Do2020}. Therefore, \eqref{eq:atsp_max_div} holds after the last step.
\end{proof}

With maximum diversity well-defined, we can determine if it is reached with population $P$ using only information from $\mathcal{N}(P)$. Therefore, we can show the guarantee of strict diversity improvement with a single 3-opt or 4-opt on a tour in some sub-optimal population, and the probability that Algorithm \ref{alg:ea} makes such an improvement, similar to Lemma \ref{lemma:ea_step_2opt_tsp} and \ref{lemma:ea_step_4opt_tsp}.

\begin{lma}\label{lemma:ea_step_3opt_atsp}
Given a population of directed tours $P$ such that $2\leq\mu\leq\left\lfloor\frac{n+2}{3}\right\rfloor$ and $d_P\geq2$, there exists a tour $I\in P$ and a triplet $(i,j,k)$ where $1\leq i<j<k\leq n$, such that $P'=(P\setminus\{I\})\cup\{3opt(I,i,j,k)\}$ satisfies \eqref{eq:improvement}. Moreover, in each iteration, the Algorithm \ref{alg:ea} with 3-opt mutation on a $(\mu,\infty)$-instance makes such an improvement with probability at least
\[\frac{3[n(d_P-2)+1][(n+1)(d_P-2)+2]}{\mu n(n-1)(n-2)(d_P-1)}.\]
\end{lma}
\begin{proof}
There must be $d_P$ tours $I$ in $P$ containing edge $e$ such that $n(e,P)=d_P$, let $I$ be one such tour. W.l.o.g, let $I$ be represented by a permutation of nodes $(i_1,i_2,\ldots,i_n)$ where $n((i_1,i_2),P)=d_P$. For all $2\leq j<k\leq n$, the operation $3opt(I,1,j,k)$ trades edges $(i_1,i_2)$, $(i_{j},i_{j+1})$, $(i_{k},i_{k+1})$ in $I$ for $(i_1,i_{j+1})$, $(i_{k},i_2)$, $(i_{j},i_{k+1})$. If for each new edge $e'$, $n(e',P)<d_P-1$, then $P'=(P\setminus\{I\}|\cup\{3opt(I,1,j,k)\}$ satisfies \eqref{eq:improvement}. We have the following equations
\begin{equation}\label{eq:sum_count_atsp}
\forall h\in[n],\sum_{j\in[n]}n((h,j),P)=\sum_{j\in[n]}n((j,h),P)=|P|=\mu.
\end{equation}
Using the Pigeonhole Principle with \eqref{eq:sum_count_atsp}, we have at least $\delta$ position $j$ from $2$ to $n-1$ such that $n((i_{1},i_{j+1}),P)<d_P-1$ where
\[\delta=n-2-\left\lfloor\frac{\mu-d_P}{d_P-1}\right\rfloor.\]
Likewise, there are at least $\delta$ positions $k$ from $3$ to $n$ such that $n((i_{k},i_{2}),P)<d_P-1$. Let the sets of such positions of $j$ and $k$ be $Q$ and $W$, respectively, $m_j=\min Q$ and $m_k=\max W$, we have that the worst-case is $m_k=m_j+2\delta-n+2$, which occurs when $Q=\{n-\delta,\ldots,n-1\}$ and $W=\{3,\ldots,\delta+2\}$. For each position $h\in W\cap(m_j,m_k]$, we define $S_h=Q\cap[m_j,h)$. We can see that for any $h\in W\cap(m_j,m_k]$, $|S_h|\geq h-m_j$. The number of choices of $j$ and $k$ in $Q$ and $W$, respectively, such that $j<k$ and $n((i_{j},i_{k+1}),P)<d_P-1$ is at least
\begin{equation}\label{eq:sigma_3opt_atsp}
\sigma=\sum_{h\in W\cap(m_j,m_k]}\max\left\lbrace|S_h|-\left\lfloor\frac{\mu-1}{d_P-1}\right\rfloor,0\right\rbrace\geq\sum_{h=1}^{2\delta-n+2}\max\left\lbrace h-\left\lfloor\frac{\mu-1}{d_P-1}\right\rfloor,0\right\rbrace.
\end{equation}
This follows from the Pigeonhole Principle, \eqref{eq:sum_count_atsp}, and $n((i_{k},i_{k+1}),P)\geq1$. We have $\sigma\geq1$ when
\begin{equation*}\label{eq:mu_bound_3opt_atsp}
\mu\leq\left\lfloor\frac{(n-2)(d_P-1)+2d_P}{3}\right\rfloor.
\end{equation*}
This proves the first part of the lemma since $d_P\geq2$. In each iteration, the Algorithm \ref{alg:ea} selects a directed tour like $I$ with probability at least $d_P/\mu$. There are at least $\sigma$ different 3-opt neighbors on such a tour to produce $P'$. Since there are $\binom{n}{3}$ 3-opt neighbors in total, the probability that the Algorithm \ref{alg:ea} obtains $P'$ from $P$ is at least
\begin{align*}
\frac{d_P}{\mu}\frac{\sigma}{\binom{n}{3}}&\geq\frac{d_P}{\mu\binom{n}{3}}\frac{1}{2}\left(n-2-2\left\lfloor\frac{\mu-d_P}{d_P-1}\right\rfloor-\left\lfloor\frac{\mu-1}{d_P-1}\right\rfloor\right)\left(n-1-2\left\lfloor\frac{\mu-d_P}{d_P-1}\right\rfloor-\left\lfloor\frac{\mu-1}{d_P-1}\right\rfloor\right)\\
&\geq\frac{d_P}{\mu\binom{n}{3}}\frac{[n(d_P-2)+1][(n+1)(d_P-2)+2]}{2(d_P-1)^2}\geq\frac{3[n(d_P-2)+1][(n+1)(d_P-2)+2]}{\mu n(n-1)(n-2)(d_P-1)},
\end{align*}
where the first inequality follows from \eqref{eq:sigma_3opt_atsp} and the second from applying the upper bound of $\mu$.
\end{proof}

We use similar approaches to derive the result for 4-opt. Some complicated expressions are encapsulated in asymptotic notations for brevity's sake.

\begin{lma}\label{lemma:ea_step_4opt_atsp}
Given a population of directed tours $P$ such that $2\leq\mu\leq\left\lfloor\frac{n}{3}\right\rfloor$ and $d_P\geq2$, there exists a tour $I\in P$ and a quadruplet $(i,j,k,h)$ where $1\leq i<j<k<h\leq n$, such that $P'=(P\setminus\{I\})\cup\{4opt(I,i,j,k,h)\}$ satisfies \eqref{eq:improvement}. Moreover, in each iteration, the Algorithm \ref{alg:ea} with 4-opt mutation on a $(\mu,\infty)$-instance makes such an improvement with probability lower-bounded by $\Omega(1/\mu n^3)$ when $d_P=2$ and $\Omega(d_P/\mu n)$ when $d_P>2$.
\end{lma}
\begin{proof}
There must be $d_P$ tours $I$ in $P$ containing edge $e$ such that $n(e,P)=d_P$, let $I$ be one such tour. W.l.o.g, let $I$ be represented by a permutation of nodes $(i_1,i_2,\ldots,i_n)$ where $n((i_1,i_2),P)=d_P$. For all $2\leq j<k<h\leq n$, the operation $4opt(I,1,j,k,h)$ trades edges $(i_1,i_2)$, $(i_{j},i_{j+1})$, $(i_{k},i_{k+1})$, $(i_{h},i_{h+1})$ in $I$ for $(i_1,i_{k+1})$, $(i_{h},i_{j+1})$, $(i_{k},i_{2})$, $(i_{j},i_{h+1})$. If for each new edge $e'$, $n(e',P)<d_P-1$, then $P'=(P\setminus\{I\})\cup\{4opt(I,1,j,k,h)\}$ satisfies \eqref{eq:improvement}. According to the Pigeonhole Principle, \eqref{eq:sum_count_atsp} and the fact $n((i_1,i_2),P)\geq1$, there are at least $\delta$ positions $k$ from $3$ to $n-1$ such that $n((i_{k},i_{2}),P)<d_P-1$ (condition 1), where
\[\delta=n-3-\left\lfloor\frac{\mu-d_P}{d_P-1}\right\rfloor.\]
Likewise, there are at least $\delta$ positions $k$ from $3$ to $n-1$ such that $n((i_{1},i_{k+1}),P)<d_P-1$ (condition 2). Let $S$ be the set of positions of $k$ such that conditions 1 and 2 hold, we have
\[|S|\geq2\delta-n+3=n-3-2\left\lfloor\frac{\mu-d_P}{d_P-1}\right\rfloor.\]
For each $k\in S$, we use the same argument to infer that there are at least $\delta_k$ positions $h\in(k,n]$ for each $j\in[2,k)$ such that $n((i_{h},i_{j+1}),P)<d_P-1$ (condition 3) and $n((i_{j},i_{h+1}),P)<d_P-1$ (condition 4), as there are at least $\delta_k'$ positions $j\in[2,k)$ for each $h\in(k,n]$, where
\[\delta_k=n-k-2\left\lfloor\frac{\mu-1}{d_P-1}\right\rfloor,\quad\text{and }\quad\delta_k'=k-2-2\left\lfloor\frac{\mu-1}{d_P-1}\right\rfloor.\]
This means the number of choices of $j$, $k$, $h$ such that conditions 1, 2, 3 and 4 hold is at least
\begin{align*}
\sigma&=\sum_{l\in S}\sigma_l=\sum_{l\in S}\max\left\lbrace(l-2)\delta_l,(n-l)\delta_l',0\right\rbrace\\&=\sum_{l\in S}\max\left\lbrace(n-l)(l-2)-2\min\{l-2,n-l\}\left\lfloor\frac{\mu-1}{d_P-1}\right\rfloor,0\right\rbrace.
\end{align*}
We have $\sigma\geq1$ if $|S|\geq1$ and $\max_{l\in S}\left\lbrace\sigma_l\right\rbrace\geq1$. The first condition is satisfied when
\[\mu\leq\left\lfloor\frac{(n-3)(d_P-1)+2d_P-1}{2}\right\rfloor.\]
The second condition is satisfied if $\Delta=|\{l=3,\ldots,n-1|\sigma_l\geq1\}|> n-3-|S|$. By solving for $\sigma_l>0$, we have
\[\Delta=\max\left\lbrace\min\left\lbrace2n-6-4\left\lfloor\frac{\mu-1}{d_P-1}\right\rfloor,n-3\right\rbrace,0\right\rbrace.\]
We can see then this condition is satisfied when
\begin{equation*}\label{eq:mu_bound_4opt_atsp}
\mu\leq\left\lfloor\frac{(n-3)(d_P-1)+d_P+1}{3}\right\rfloor.
\end{equation*}
This proves the first part of the lemma since $d_P\geq2$. Since $\mu\in[2,\lfloor n/3\rfloor]$, we can assume $n\geq6$. Let $S'=\{l\in S|\sigma_l\geq1\}$ and $B=\{l=3,\ldots,n-1|\sigma_l\in(0,\sigma_{(n+2)/2})\}$, we have
\[|B|=\max\left\lbrace2\left(\left\lfloor\frac{n-3}{2}\right\rfloor-2\left\lfloor\frac{\mu-1}{d_P-1}\right\rfloor\right),0\right\rbrace,\text{ and }|S'|-|B|\geq\min\left\lbrace n-1-2\left\lfloor\frac{\mu-1}{d_P-1}\right\rfloor,|S'|\right\rbrace\geq1.\]
We have $|B|>0$ if $d_P\geq\frac{4\mu-3}{n-4}+1\geq2$. Using the fact that $\sigma_l$ is a piece-wise quadratic function, we get
\begin{equation}\label{eq:sigma_4opt_atsp}
\sigma\geq2\sum_{l=1}^{|B|/2}\sigma_{l+2}+4\sum_{l=|B|/2+1}^{\left\lfloor\frac{|S'|+|B|}{4}\right\rfloor}\sigma_{l+2}\geq\begin{cases}
\frac{2n-4}{8}&\text{if }d_P=2\text{ and }\mu=\frac{n}{3}\\
\Omega(n^3)&\text{if }d_P>2\text{ and }\mu=\frac{n}{3}
\end{cases}.
\end{equation}

In each iteration, the Algorithm \ref{alg:ea} selects a directed tour like $I$ with probability at least $d_P/\mu$. There are at least $\sigma$ different 4-opt neighbors on such a tour to produce $P'$. Since there are $\binom{n}{4}$ 4-opt neighbors in total, the probability that the Algorithm \ref{alg:ea} obtains $P'$ from $P$ is at least
\begin{align*}
\frac{d_P}{\mu}\frac{\sigma}{\binom{n}{4}}\geq
\begin{cases}
\frac{12(n-2)}{\mu n(n-1)(n-2)(n-3)}&\text{if }d_P=2\\
\Omega(d_P/\mu n)&\text{if }d_P>2
\end{cases},
\end{align*}
following from \eqref{eq:sigma_4opt_atsp}.
\end{proof}

Compared to Lemma \ref{lemma:ea_step_2opt_tsp} and \ref{lemma:ea_step_4opt_tsp}, the upper bounds of $\mu$ that guarantee lack of local optima are looser in Lemma \ref{lemma:ea_step_3opt_atsp} and \ref{lemma:ea_step_4opt_atsp}, in exchange for lower improvement probability bounds. The former is due to the greater numbers of operation choices: $\mathcal{O}(n^3)$ in 3-opt and $\mathcal{O}(n^4)$ in 4-opt, compared to $\mathcal{O}(n^2)$ in 2-opt and 4-opt exchange. This creates more flexibility, making it more likely to be able to escape local optima. The latter is due to the fact that with $\mu$ being at the upper bound, the numbers of satisfactory operations remain constant (or linear w.r.t. $n$ in case of 4-opt), leading to smaller improvement probabilities as the total numbers of operations increase. We can see that this would lead to a greater expected asymptotic run-time.

\begin{thm}\label{theo:runtime_atsp}
On a $(\mu,\infty)$-instance based on any ATSP instance with $n\geq6$ nodes, and $\mu\geq2$, the Algorithm \ref{alg:ea} obtains a $\mu$-population with maximum diversity within expected time $\mathcal{O}(\mu^2n^4)$ if
\begin{itemize}
\item it uses 3-opt mutation and $\mu\leq\left\lfloor\frac{n+2}{3}\right\rfloor$,
\item it uses 4-opt mutation and $\mu\leq\left\lfloor\frac{n}{3}\right\rfloor$.
\end{itemize}
\end{thm}
\begin{proof}
In the worst case, the algorithm begins with $d_P=\mu$ and $c_P=n$. At any time, we have $c_P\leq\mu n/d_P$. Moreover, in the worst case, each improvement either reduces $c_P$ by $1$, or reduces $d_P$ by $1$ and sets $c_P$ to its maximum value. With $2\leq\mu\leq n-1$, the maximum diversity is achieved iff $d_P=1$ since the construction in \cite{Do2020} can be extended to directed complete graphs by creating two directed tours out of each undirected tour. According to Lemma \ref{lemma:ea_step_3opt_atsp}, the expected run time Algorithm \ref{alg:ea} requires to reach maximum diversity when using 3-opt mutation is at most
\[\sum_{j=2}^{\mu}\frac{\mu n}{j}\frac{\mu n(n-1)(n-2)(j-1)}{3[n(j-2)+1][(n+1)(j-2)+2]}=\mathcal{O}(\mu^2n^4).\]
Similarly, Lemma \ref{lemma:ea_step_4opt_atsp} implies that when $2\leq\mu\leq\left\lfloor\frac{n}{3}\right\rfloor$, Algorithm \ref{alg:ea} with 4-opt mutation needs at most the following expected run time
\[\frac{\mu^2 n^2(n-1)(n-2)(n-3)}{24(n-2)}+\sum_{j=3}^{\mu}\frac{\mu n}{j}\mathcal{O}(\mu n/j)=\mathcal{O}(\mu^2n^4).\]
\end{proof}

\subsection{QAP}\label{sec:qap}

In QAP, we largely use the same notations as we define in Section \ref{sec:tsp}. Let $d_P=\max_{i,j\in [n]}\{n(i,j,P)\}$ and $c_P=|{i,j\in [n]\mid n(i,j,P)=d_P}|$. For convenience, we use the notation $A(P)=\{(i,j)\mid\exists a\in P,a(i)=j\}$. Let $\phi$ be a shift operation such that for all permutation $a:[n]\to[n]$,
\[b=\phi(a)\implies\forall i\in[n-1],b(i)=a(i+1)\wedge b(n)=a(1).\]
We first show the achievable maximum diversity for any positive $n$ and $\mu$, which will be the foundation for our run-time analysis.

\begin{thm}\label{theo:qap_max_div}
Given $n,\mu\geq1$, there exists a $\mu$-size population $P$ of permutations of $[n]$ such that
\begin{equation}\label{eq:qap_max_div}
\max_{i,j\in [n]}n(i,j,P)-\min_{i,j\in [n]}n(i,j,P)\leq1.
\end{equation}
\end{thm}
\begin{proof}
We prove by constructing such a $P$. Let $a:[n]\to[n]$ be some arbitrary permutation and $Q=\{\phi^i(a)\mid i\in[n]\}$ where $\phi^i$ is $\phi$ applied $i$ times. Note that $\phi^n(a)=a$. It is the case that no two solutions in $Q$ share assignments, so for all $i,j\in [n]$, we have $n(i,j,Q)=1$, and $A(Q)=[n]\times[n]$. Let $\mu=kn+r$ where $k,r\in\mathbb{N}$ and $r<n$, and $B\subset Q$ where $|B|=r$, we include in $P$ $k+1$ copies of each solution in $B$ and $k$ copies of each solution in $Q\setminus B$. Then $P$ satisfies \eqref{eq:qap_max_div} since
\[\forall(i,j)\in A(B),n(i,j,P)=k+1,\text{ and }\forall(i,j)\in A(Q\setminus B),n(i,j,P)=k.\]
\end{proof}

Here, we give a formal definition of k-opt. We denote the k-opt transformation by $s(S,p,\cdot)$ where $(S,p)\in\{Z\subseteq[n]\mid|Z|=k\}\times\{q:[k]\to[k]\mid\forall i\in[k],q(i)\neq i\}$. The operation is defined on a permutation $a$ as follow
\begin{align}\label{eq:define_kopt_qap}
&a'=s(S,p,a)\implies a'(i)=\begin{cases}
a(i)&\text{if }i\notin S\\
(a\circ r_S^{-1}\circ p\circ r_S)(i)&\text{otherwise}
\end{cases},\\&r_S:S\to[k],\forall i,j\in S,i>j\iff r_S(i)>r_S(j),\nonumber
\end{align}

where $\circ$ denotes a function composition. The operation modifies exactly $k$ positions, $S$, in the permutation, by shuffling elements in those positions according to a derangement $p$. There are $!k=\left\lfloor\frac{k!+1}{e}\right\rfloor$ derangements on $k$ positions \cite{hassani2003derangements}, and $!k\binom{n}{k}$ distinct k-opt neighbors of a permutation on $[n]$. We illustrate this definition with an example
\[S=\{1,3,5\},p=(2,3,1)\implies s(S,p,(\mathcolor{red}{5},4,\mathcolor{red}{3},2,\mathcolor{red}{1}))=(\mathcolor{red}{3},4,\mathcolor{red}{1},2,\mathcolor{red}{5}).\]

Again, we establish a lower bound of the probability of strict progressions toward maximum diversity, in order to obtain a worst-case run-time result. For brevity's sake, we reuse the expression \eqref{eq:improvement} with notations defined in the QAP context. Here, $\mathbbm{1}_F$ denotes a characteristic function assuming value $1$ if the logical expression $F$ holds, and $0$ otherwise.

\begin{lma}\label{lemma:ea_step_kopt_qap}
Given $2\leq k\leq n-1$ and a population of $[n]\to[n]$ permutations $P$ such that $2\leq\mu\leq\left\lfloor\frac{n-k+3+\mathbbm{1}_{k=2}}{2}\right\rfloor$ and $d_P\geq2$, there exists a permutation $a\in P$, a $k$-subset $S$ of $[n]$ and a derangement $p:[k]\to[k]$, such that $P'=(P\setminus\{a\})\cup\{s(S,p,a)\}$ satisfies \eqref{eq:improvement}. Moreover, in each iteration, the Algorithm \ref{alg:ea} with k-opt mutation on a $(\mu,\infty)$-instance makes such an improvement with probability at least
\[\left(\frac{d_P-1.5}{d_P-1}\right)^{k-2}\frac{(n-k+2+\mathbbm{1}_{k=2})(d_P-2)+1}{\mu (n-1)(n-k+1+\mathbbm{1}_{k=2})(!k)/k!}.\]
\end{lma}
\begin{proof}
There must be $d_P$ permutations $a$ in $P$ such that $\exists i\in [n],n(i,a(i),P)=d_P$, let $a$ be one such permutation, and $i\in[n]$ such that $n(i,a(i),P)=d_P$. The operation $s(S,p,a)$, for each $j\in S$, removes assignments $j\to a(j)$, and adds $j\to (a\circ r_S^{-1}\circ p\circ r_S)(j)$, with $r_S$ defined in \eqref{eq:define_kopt_qap}. If the counts of all new assignments are less than $d_P-1$, then  $P'=(P\setminus\{a\})\cup\{s(S,p,a)\}$ satisfies \eqref{eq:improvement}. We show that there is always such a pair $(S,p)$ by constructing them step-by-step, starting from $i$ (i.e. $S_0=\{i\}$) and an empty permutation $q$. We will frequently make use of the following equations
\begin{equation}\label{eq:sum_count_qap}
\forall h\in[n],\sum_{j\in[n]}n(h,j,P)=\sum_{j\in[n]}n(j,h,P)=|P|=\mu.
\end{equation}

Let $Z_1=\{x\in[n]\setminus S_0\mid n(i,a(x),P)<d_P-1\}$, for the first step, we can add an element $x'\in Z_1$, $S_{1}=S_0\cup\{x'\}$ and set $q(i)=x'$. Using the Pigeonhole principle with \eqref{eq:sum_count_qap} gives
\[|Z_1|\geq n-1-\left\lfloor\frac{\mu-d_P}{d_P-1}\right\rfloor.\]

At the $l$-th step for all $2\leq l\leq k-2$, let $Z_l=\{x\in[n]\setminus S_{l-1}\mid n(y_{l-1},a(x),P)<d_P-1\}$ where $y_{l-1}$ is the element added in the previous step, we can add an element $y_l\in Z_l$, $S_{l}=S_{l-1}\cup\{y_l\}$, and set $q(y_{l})=y_{l+1}$. Using the Pigeonhole principle with \eqref{eq:sum_count_qap} and the fact $n(y_{l-1},a(y_{l-1}),P)\geq1$ gives
\begin{equation}\label{eq:no_choice_l_qap}
|Z_l|\geq n-l-\left\lfloor\frac{\mu-1}{d_P-1}\right\rfloor.
\end{equation}

At the $k-1$-th step, let $Z_{k-1}$ be defined similarly as $Z_l$, and $B=\{x\in[n]\setminus S_{k-2}\mid n(x,i,P)<d_P-1\}$, we can add an element $y_{k-1}\in Z_{k-1}\cap B$, $S=S_{k-2}\cup\{y_{k-1}\}$, and set $p(y_{k-2})=y_{k-1}$ and $p(y_{k-1})=i$. Since the inequality \eqref{eq:no_choice_l_qap} also holds for $l=k-1$, we have
\[|Z_{k-1}\cap B|\geq|Z_{k-1}|+|B|-n+k-1\geq n-k+1-\left\lfloor\frac{\mu-d_P}{d_P-1}\right\rfloor-\left\lfloor\frac{\mu-1}{d_P-1}\right\rfloor\geq n-k-\left\lfloor2\frac{\mu-d_P}{d_P-1}\right\rfloor.\]

Finally, we define $p=r_S\circ q\circ r_S^{-1}$, thus finish constructing $S$ and $p$. According to each step, for all $j\in S$, $n(j,(a\circ q)(j),P)<d_P-1$. Since $q=r_S^{-1}\circ p\circ r_S$, $s(S,p,a)$ is the desirable operation. As the sequence $(y_j)_{j=1}^{k-1}$ uniquely defines the pair $(S,p)$, we have the minimum number of desirable k-opt neighbors
\begin{equation}\label{eq:min_kopt_qap}
\sigma_k=|Z_{k-1}\cap B|\prod_{l=1}^{k-2}|Z_l|\geq\left(n-k-\left\lfloor2\frac{\mu-d_P}{d_P-1}\right\rfloor\right)\left(n-1-\left\lfloor\frac{\mu-d_P}{d_P-1}\right\rfloor\right)\prod_{l=2}^{k-2}\left(n-l-\left\lfloor\frac{\mu-1}{d_P-1}\right\rfloor\right).
\end{equation}
Note that \eqref{eq:min_kopt_qap} only applies to $k>2$. For $k=2$, there is only one step, so we instead have
\[\sigma_2=|Z_1\cap B|\geq n-1-2\left\lfloor\frac{\mu-d_P}{d_P-1}\right\rfloor.\]
In any case, $\sigma_k>0$ if $|Z_{k-1}\cap B|\geq1$, which is satisfied when
\[\mu\leq\left\lfloor\frac{(n-k+\mathbbm{1}_{k=2})(d_P-1)+2d_P-1}{2}\right\rfloor.\]

This proves the first part of the lemma since $d_P\geq2$. In each iteration, the Algorithm \ref{alg:ea} selects a mapping like $a$ with probability at least $d_P/\mu$. There are at least $\sigma_k$ different k-opt neighbors on such a mapping to produce $P'$. Since there are $!k\binom{n}{k}$ k-opt neighbors in total, for $k>2$, the probability that the Algorithm \ref{alg:ea} obtains $P'$ from $P$ is at least
\begin{align*}
\frac{\sigma_k}{\binom{n}{k}!k}\frac{d_P}{\mu}&\geq\frac{d_P}{\mu}\frac{\left(n-k-\left\lfloor2\frac{\mu-d_P}{d_P-1}\right\rfloor\right)\left(n-1-\left\lfloor\frac{\mu-d_P}{d_P-1}\right\rfloor\right)\prod_{l=2}^{k-2}\left(n-l-\left\lfloor\frac{\mu-1}{d_P-1}\right\rfloor\right)}{!k\prod_{j=0}^{k-1}(n-j)/k!}\\
&
\geq\left(\frac{d_P-1.5}{d_P-1}\right)^{k-2}\frac{(n-k+2)(d_P-2)+1}{\mu (n-1)(n-k+1)(!k)/k!}.
\end{align*}
For $k=2$, this lower bound is
\[\frac{\sigma_2}{\binom{n}{2}!2}\frac{d_P}{\mu}\geq\frac{(n+1)(d_P-2)+1}{d_P-1}\frac{2d_P}{\mu n(n-1)}\geq\frac{2[(n+1)(d_P-2)+1]}{\mu n(n-1)}.\]
\end{proof}

Regarding the upper bound of $\mu$, we can see that the presence of $k$, that is the number of elements affected by mutation, in the numerator exhibits similar pattern as in the bounds of $\mu$ in Lemma \ref{lemma:ea_step_3opt_atsp} and \ref{lemma:ea_step_4opt_atsp}. The difference is the denominator, which is $3$ instead $2$ as in Lemma \ref{lemma:ea_step_kopt_qap}. This might be explained by the observation that the minimum edge distance between two directed tours is $3$ and not $2$ as between permutations.

It is important to note that in Lemma \ref{lemma:ea_step_kopt_qap}, we only consider one scenario where the improvement can be made, from which the upper bound of $\mu$ is derived. One would then assume that other scenarios would make strict improvements possible at larger $\mu$, even at near-optimal diversity (e.g. $d_P=2$). It turns out that this is not the case with 2-opt, meaning if $\mu$ exceeds this bound, then no improvement scenario is guaranteed. We demonstrate the tightness of this bound with the following constructive proof.

\begin{pro}\label{lemma:ea_step_2opt_qap_bound}
Given $n\geq5$, there exists a population $P$ of $\left\lfloor\frac{n+2}{2}\right\rfloor+1$ permutations on $[n]$ such that $d_P=2$ and for all 2-opt $s(S,p,\cdot)$ and $a,b\in P$, the new population $P'=(P\setminus\{a\})\cup\{s(S,p,b)\}$ is such that $\mathcal{D}(P')\geq\mathcal{D}(P)$.
\end{pro}
\begin{proof}
We prove by construction. For convenience, let $l=\left\lfloor\frac{n+2}{2}\right\rfloor$. Firstly, we see that if $a\neq b$, then both $b$ and $s(S,p,b)$ are in $P'$. Since $|A(b)\cap A(s(S,p,b))|=n-2$ for any 2-opt $s(S,p,\cdot)$, we know that regardless of $P$, $d_{P'}\geq2$ and if $d_{P'}=2$ then $c_{P'}\geq n-2$. We can construct $l+1$ permutations for $P$ by adding an arbitrary starting permutation, then sequentially applying the shift operation $\phi$ to generate $l$ more, and finally applying 2-opt on the last permutation at any 2 cyclically consecutive positions. This gives us $P$ such that $d_P=2$ and $c_P=1$, meaning $\mathcal{D}(P')\geq\mathcal{D}(P)$. Therefore, we assume $a=b$. Secondly, we observe that $\mathcal{D}(P')<\mathcal{D}(P)$ only holds if $b$ (the permutation undergoing the 2-opt) is such that there is a position $i\in[n]$ where $n(i,b(i),P)=d_P$; the 2-opt must also change said position.

For $n=6$, we have $a_1=(i_1,\ldots,i_6)$, $a_2=\phi^{2}(a)$, $a_3=(i_4,i_5,i_6,i_2,i_3,i_1)$, $a_4=(i_2,i_1,i_4,i_5,i_6,i_3)$, and $a_5=(i_1,i_6,i_2,i_3,i_4,i_5)$. For any even $n>6$, let $a_1=(i_1,\ldots,i_n)$ w.l.o.g and for $j=2,\ldots,l-1$ let $a_j=s_{n-j+1}\circ\phi^{j}(a_1)$ where $s_j$ is a 2-opt in positions $j$ and $j+1$. Finally, let
\[a_l=(i_2,i_1,i_4,\ldots,i_{n/2+1},i_3,i_{n/2+3},\ldots,i_{n},i_{n/2+2}),\quad a_{l+1}=(i_1,i_n,i_2,\ldots,i_{n-1}),\]
and $P=\{a_1,\ldots,a_{l+1}\}$. We have $|A(a_1)\cap A(a_{l+1})|=1$ and $A(a_1)\cap A(a_j)=A(a_{l+1})\cap A(a_j)=A(a_h)\cap A(a_j)=\emptyset$ for any $j,h=2,\ldots,l$ and $j\neq h$. For any $j=2,\ldots,n/2+1$, $n(1,i_j,P)=1$, and for any $j=n/2+1,\ldots,n$, $n(j,i_1,P)=1$. This means \eqref{eq:improvement} cannot be satisfied from any 2-opt on $a_1$, nor can it be satisfied from a 2-opt on $a_{l+1}$ at positions $1$ and $j$ for any $j=3,\ldots,n$. Given that $n(2,i_1,P)=1$ since $(2,i_1)\in A(a_l)$, we have that $\mathcal{D}(P')\geq\mathcal{D}(P)$ by any 2-opt on any permutation in $P$.

For any odd $n\geq5$, again let $a_1=(i_1,\ldots,i_n)$, and for all $j=2,\ldots,l-1$, $a_j=\phi^j(a_1)$. Finally,
\begin{align*}
&a_l=(i_2,i_{(n+5)/2},\ldots,i_n,i_{(n+3)/2},i_3,\ldots,i_{(n+1)/2},i_1),\\\text{and}\quad&a_{l+1}=(i_1,i_3,\ldots,i_{(n+1)/2},i_2,i_{(n+5)/2},\ldots,i_n,i_{(n+3)/2}).
\end{align*}
Similarly, $\mathcal{D}(P')<\mathcal{D}(P)$ only if the 2-opt is performed on $a_1$ or $a_{l+1}$, and changes position $1$ of either. However, such an operation must introduce an assignment in $\{(1,i_j),(j+(n-1)/2,i_1)\mid j=2,\ldots,(n+1)/2\}$, all of which have count $1$ in $P$. Since $d_P=2$ and $c_P=1$, this means $\mathcal{D}(P')\geq\mathcal{D}(P)$.
\end{proof}

Naturally, Lemma \ref{lemma:ea_step_kopt_qap} allows us to derive the following run-time bound for Algorithm \ref{alg:ea}, similar to Theorem \ref{theo:runtime_tsp} and \ref{theo:runtime_atsp}.

\begin{thm}\label{theo:runtime_qap}
On a $(\mu,\infty)$-instance based on any QAP instance with $n\geq1$, and $2\leq\mu\leq\left\lfloor\frac{n-k+3+\mathbbm{1}_{k=2}}{2}\right\rfloor$, the Algorithm \ref{alg:ea} with k-opt mutation obtains a $\mu$-population with maximum diversity within expected time $\mathcal{O}(2^{k-2}\mu^2n^2(n-k+1))$.
\end{thm}
\begin{proof}
In the worst case, the algorithm begins with $d_P=\mu$ and $c_P=n$. At any time, we have $c_P\leq\mu n/d_P$. Moreover, in the worst case, each improvement either reduces $c_P$ by $1$, or reduces $d_P$ by $1$ and sets $c_P$ to its maximum value. With $2\leq\mu\leq\left\lfloor\frac{n-k+3+\mathbbm{1}_{k=2}}{2}\right\rfloor$, the maximum diversity is achieved iff $d_P=1$ according to Theorem \ref{theo:qap_max_div}. According to Lemma \ref{lemma:ea_step_kopt_qap}, the expected run time Algorithm \ref{alg:ea} requires to reach maximum diversity is at most
\[\sum_{j=2}^{\mu}\frac{\mu n}{j}\left(\frac{j-1}{j-1.5}\right)^{k-2}\frac{\mu (n-1)(n-k+1+\mathbbm{1}_{k=2})(!k)/k!}{(n-k+2+\mathbbm{1}_{k=2})(j-2)+1}=\mathcal{O}(2^{k-2}\mu^2n^2(n-k+1)).\]
\end{proof}

The results in Theorem \ref{theo:runtime_tsp} and \ref{theo:runtime_qap} (for $k=2$) are identical due to similarities between structures of STSP tours and QAP mappings, and the same intuition applies. Of note is that according to the proofs, the probability of making improvements drops as the population is closer to maximum diversity. This is a common phenomenon for randomized heuristics in general, which we expect to see replicated in experimentation. Our result also shows that such reduction is more severe at stronger mutation strengths, which seems to be the direct consequence of greater flexibility that comes with making larger changes. Lemma \ref{lemma:ea_step_kopt_qap} implies that this also holds in scenarios where $\alpha$ is non-trivial in a sense that it changes how diverse the population can be; the algorithm tends to have a harder time converging with stronger mutations, when the population is close to maximally achievable diversity.

\section{Experimental investigations}\label{sec:experiment}

We perform two sets of experiments to establish baseline results for evolving diverse QAP mappings. These involve running Algorithm \ref{alg:ea} separately using two described measures: $\mathcal{N}$ \eqref{eq:N_div_TSP} and $\mathcal{D}$ \eqref{eq:D_div}. We denote these two variants by $D_1$ and $D_2$ as they are correspondingly correlated measures. The mutation operator used is 2-opt. Firstly, we consider the unconstrained case where no quality constraint is applied. Then, we impose constraints with varying quality thresholds $\alpha$ on the solutions. For similar experiments on TSP, we refer to \cite{Do2020}.

For our experiments, we use three QAPLIB instances: Nug30 \cite{Nugent1968}, Lipa90b \cite{Li1992}, Esc128 \cite{Eschermann}. The optimal solutions for these instances are known\footnote{The QAPLIB instances are publicly available at \url{https://coral.ise.lehigh.edu/data-sets/qaplib/}}. We vary the population size among $3$, $10$, $20$, $50$. We run each variant of the algorithm 30 times on each instance, and each run is allotted $\mu n^2$ maximum iterations. It is important to note that any reported diversity score is normalized with the upper bound appropriate to the instance. For $D_1$, the bound is derived from Theorem \ref{theo:qap_max_div}, while it is $\mu n$ for $D_2$ as mentioned. We specify the differences in settings between unconstrained case and constrained case in the following sections.

\subsection{Unconstrained diversity optimization}\label{sec:unconstrained}

In the unconstrained case, we are interested in how optimizing for one measure affect the other, and how many iterations are needed to reach maximum diversity from zero diversity. To this end, we set the initial population to contain only duplicates of some random tour. Furthermore, we apply a stopping criterion that holds when the measure being optimized for reaches its upper bound. However, for $n>\mu$, the bound is unreachable, so we expect that the algorithm does not terminate prematurely while minimizing $\mathcal{D}$.

\begin{figure}[t]
\centering
\includegraphics[width=1\linewidth]{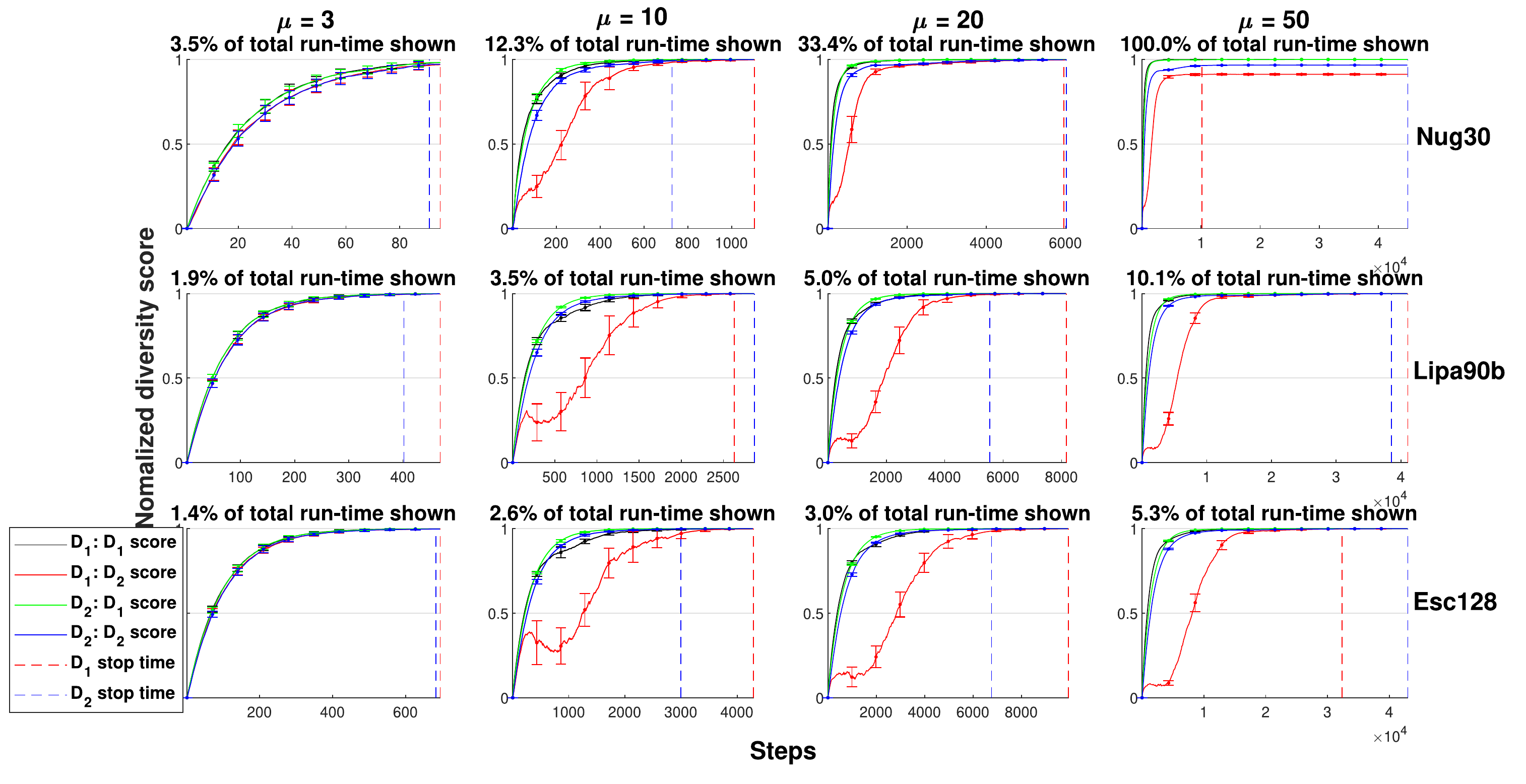}
\caption{Means and standard deviations of normalized $D_1$ and $D_2$ scores from both approaches over time. The total run-time is $\mu n^2$. The dashed lines denote the average numbers of steps till termination.
}
\label{fig:qap_unconstrained}
\end{figure}

Figure \ref{fig:qap_unconstrained} shows the mean diversity scores and their standard deviations throughout the runs, and the average numbers of iterations till termination. Each column corresponds to a $\mu$ value, and each row a QAPLIB instance. For visibility, in each case, the X-axis range is scaled to the maximum number of steps till termination from all runs, and missing data points are extrapolated from the final scores. Since the red curves exhibit extremely noisy behaviors, smoothing is applied to expose the high-level trends, by taking an average in each of 500 equal-length intervals along the time dimension. This changes the appearance of other curves very minimally. Overall, when $\mu\leq n$, Algorithm \ref{alg:ea} maximizes both $D_1$ and $D_2$ well within the run time limit. The ratios between needed run-times and corresponding total run-times seem to correlate with the ratio $\mu/n$. Additionally, the algorithm seems to require similar run-time to optimize for both measures, as no consistent differences are visible.

The figure also shows a notable difference in the evolutionary trajectories resulted from using $\mathcal{N}$ and $\mathcal{D}$ for survival selection. When $\mathcal{D}$ is used, Algorithm \ref{alg:ea} improves $D_1$ about as efficiently as when $\mathcal{N}$ is used. On the other hand, when $\mathcal{N}$ is used, it increases $D_2$ poorly during the early stages in many cases, and in some cases even noticeably decreases it in short periods. Furthermore, in many cases, $D_2$ only starts to increase quickly when $D_1$ reaches a certain threshold. That said, this particular difference is not observable for $\mu=3$. Nevertheless, it indicates that even in easy cases ($\mu\leq n$), highly even distributions of assignments in the population are unlikely to prevent clustering. In fact, judging by the noisy behaviors in the red curves, the degree of clustering seems almost uncorrelated to $D_1$. In contrast, separating each solution from the rest of the population tends to improve overall diversity effectively.

To further investigate the impact of $n$ and $\mu$ on the run-time of Algorithm \ref{alg:ea}, we carry out another experiment on synthetic instances with exhaustive combinations of $(n,\mu)$ values. More precisely, $n$ is assigned values from $20$ to $120$ with step $5$, and $\mu$ from $5$ to $120$ with step $5$. For each value pair of $(n,\mu)$, we run 30 times Algorithm \ref{alg:ea} using $\mathcal{N}$, with $\mu n^2$ maximum iterations, and record the number of step it takes to reach maximum diversity in each run. If it fails, the maximum iteration is recorded. Additionally, we experiment with different mutation strengths, namely 2-opt, 3-opt, 4-opt, and $\lceil n/5\rceil$-opt.

\begin{figure}[ht]
\centering
\begin{subfigure}[b]{0.49\textwidth}
\centering
\includegraphics[width=1\linewidth]{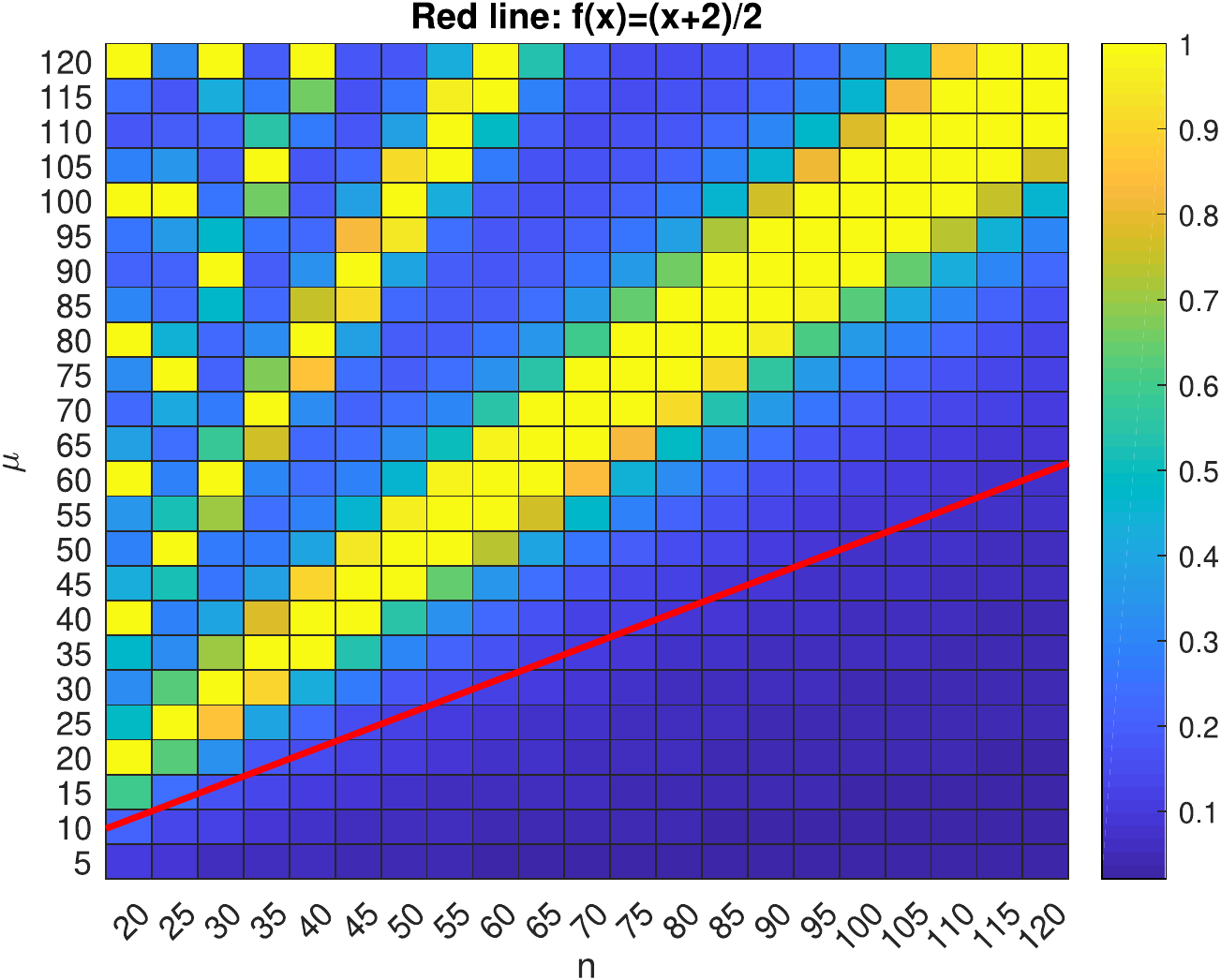}
\caption{2-opt}\label{fig:qap_unconstrained21}
\end{subfigure}
\begin{subfigure}[b]{0.49\textwidth}
\centering
\includegraphics[width=1\linewidth]{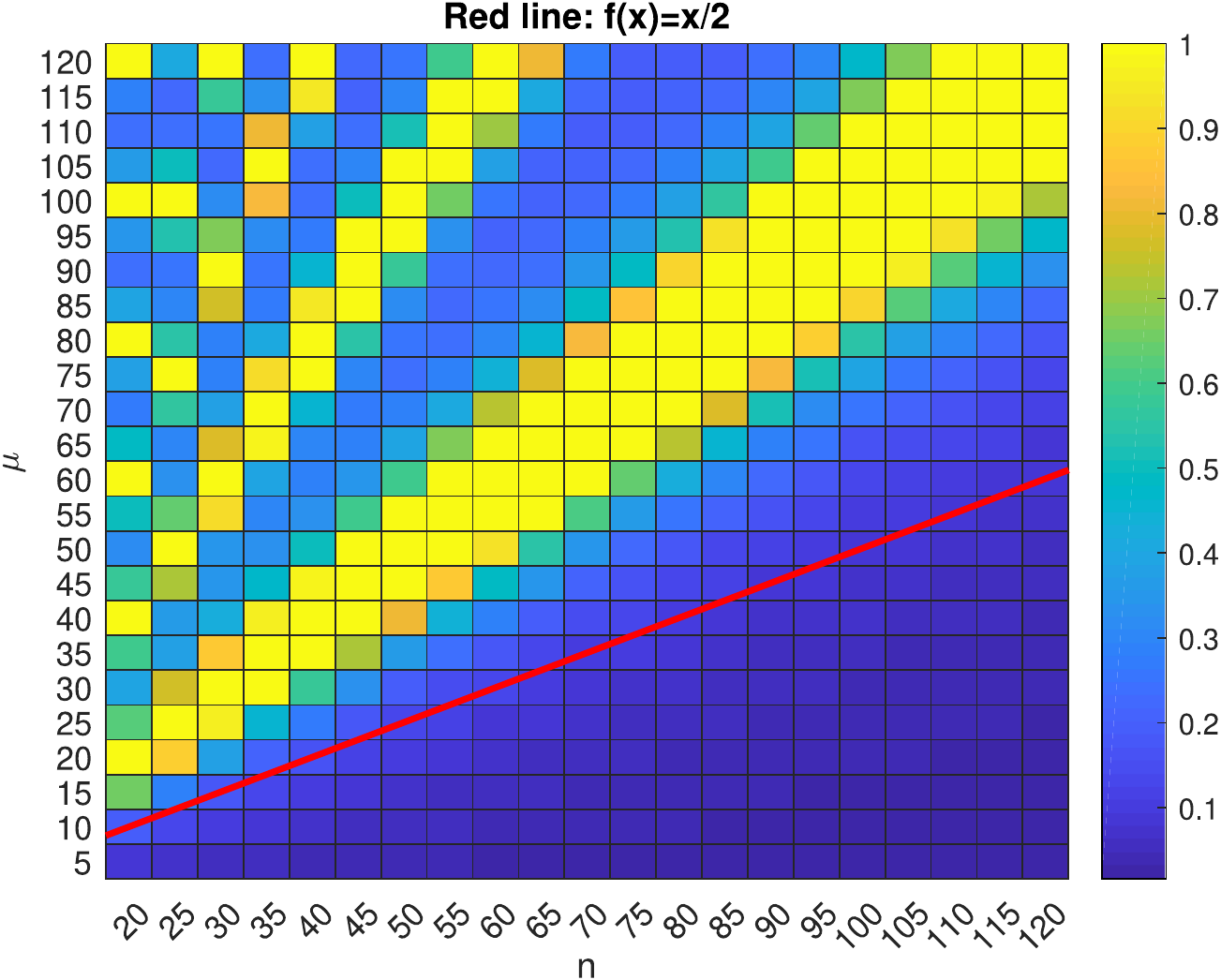}
\caption{3-opt}\label{fig:qap_unconstrained22}
\end{subfigure}
\begin{subfigure}[b]{0.49\textwidth}
\centering
\includegraphics[width=1\linewidth]{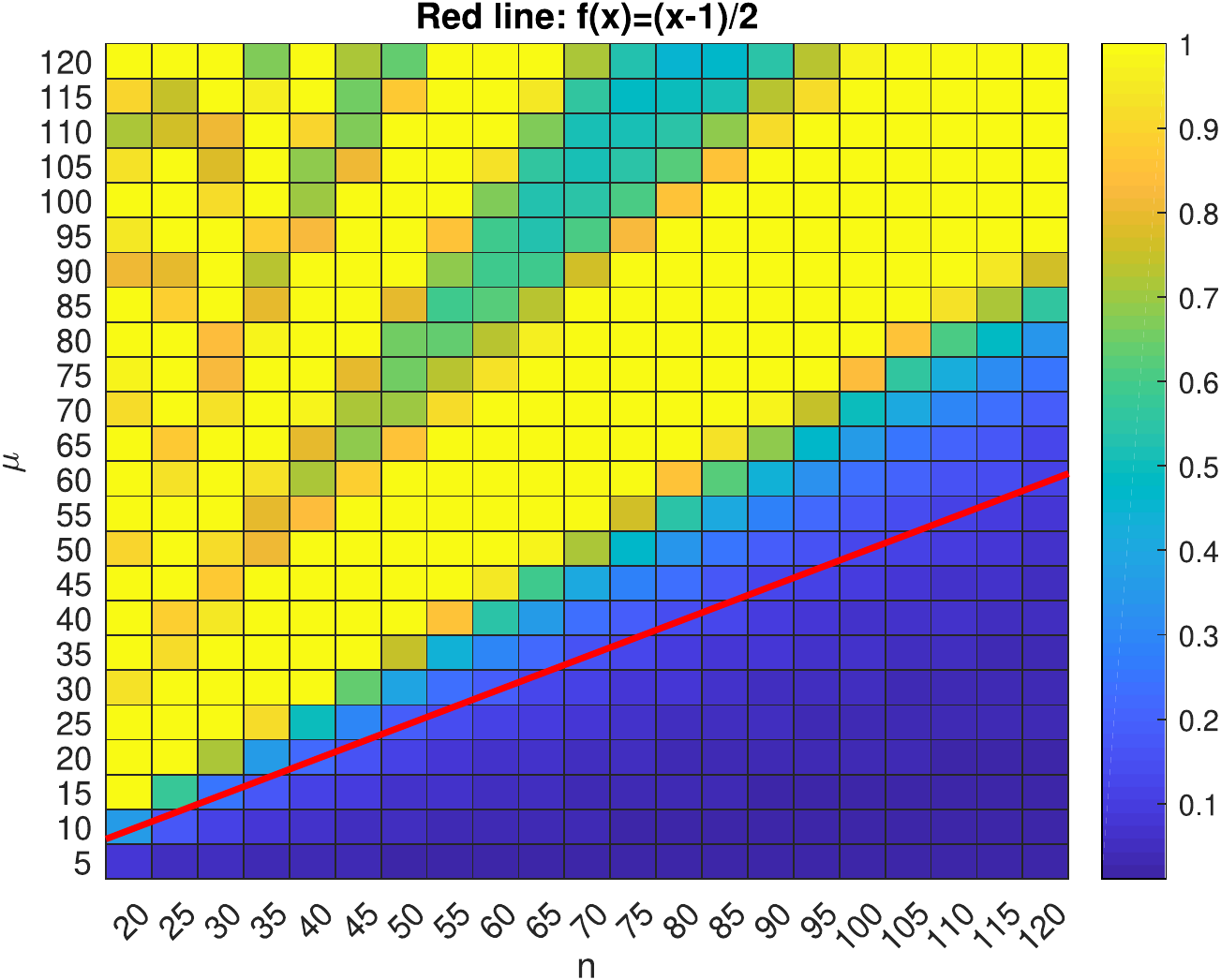}
\caption{4-opt}\label{fig:qap_unconstrained23}
\end{subfigure}
\begin{subfigure}[b]{0.49\textwidth}
\centering
\includegraphics[width=1\linewidth]{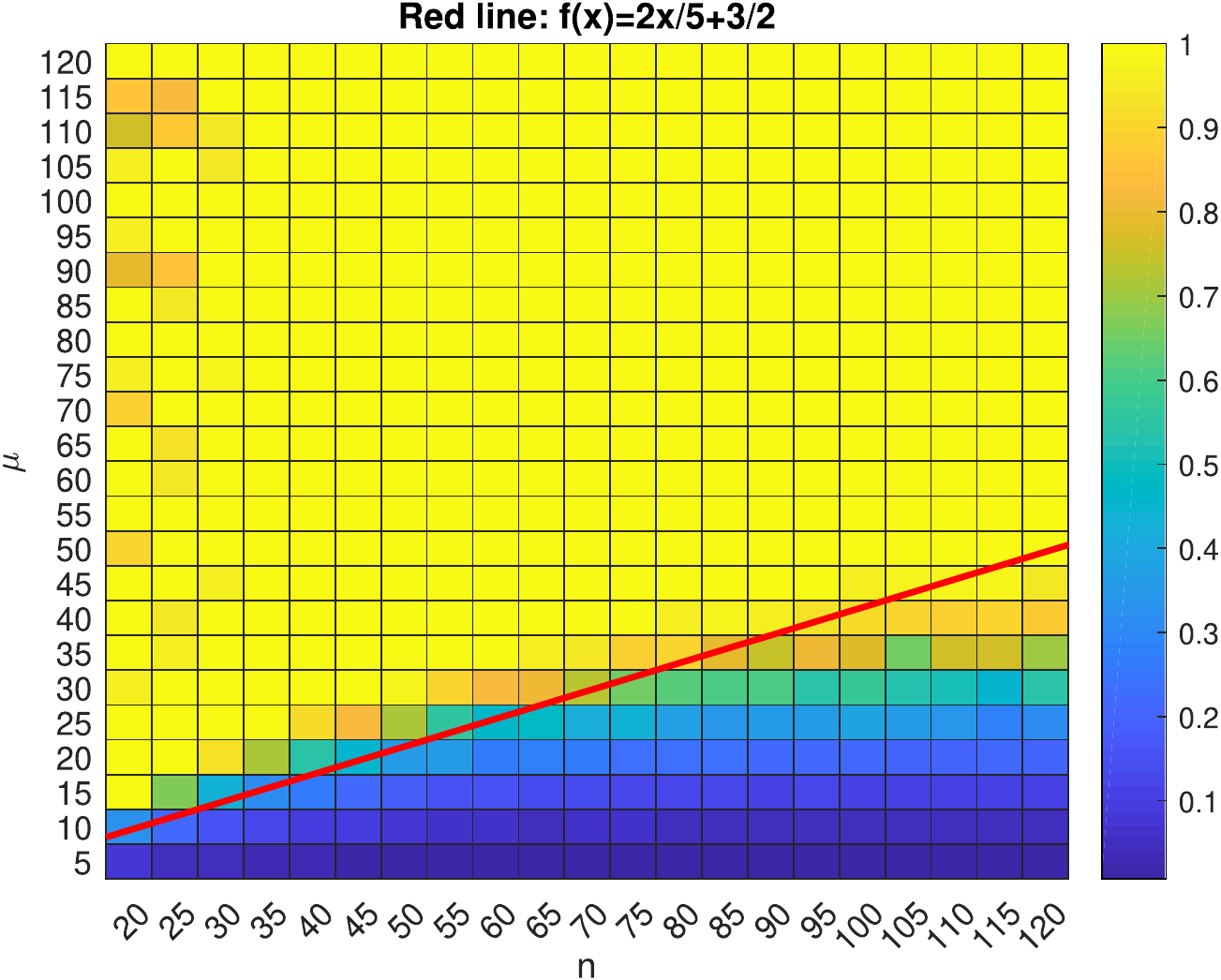}
\caption{$\lceil n/5\rceil$-opt}\label{fig:qap_unconstrained24}
\end{subfigure}
\caption{Heat-maps of mean time-to-optimum as percentages of the respective budgets $\mu n^2$. The red lines indicate the upper bound of the $\mu$ values to which the result of Theorem \ref{theo:runtime_qap} applies.
}
\label{fig:qap_unconstrained2}
\end{figure}

Figure \ref{fig:qap_unconstrained2} shows the 2D heat-maps of recorded numbers of steps as percentages of the corresponding iteration budgets. Each axis (x-axis for $n$, y-axis for $\mu$) is discretized into bins, each of which corresponds to a specific parameter value used in the experiment; each cell corresponding to a pair of values is assigned a color denoting the percentage. Dark colors indicate cases where the algorithm efficiently converges to maximum diversity, whereas bright colors suggest difficult cases. Heat-map \ref{fig:qap_unconstrained21} shows mostly dark cells, with bright cells only concentrated in areas where $\mu$ is multiples of $n$. Heat-map \ref{fig:qap_unconstrained21} exhibits a similar pattern, differing only in having more bright cells around these regions. Heat-maps \ref{fig:qap_unconstrained23} and \ref{fig:qap_unconstrained24} extend this trend further, with \ref{fig:qap_unconstrained24} showing mostly bright cells, and only relatively few dark cells at small $\mu$ values.

From Figure \ref{fig:qap_unconstrained2}, we can observe that the algorithm terminates well within the budget, except for when $\mu$ is close to multiples of $n$. In these cases, it frequently gets stuck in near-optimal diversity states for most of the budget. Note that the average required run-time is small even at larger $\mu$ and $n$, as long as $\mu$ is sufficiently far from multiples of $n$. This suggests that in such cases, the search trajectories made by incremental improvements almost guarantee reaching hard-to-escape local optima. The heat-maps suggest that these detrimental cases occur when $\mu\in[(t-\epsilon)n,(t+\epsilon)n]$, where $t\in\mathbb{N}^*$ and $\epsilon>0$. Assuming this, $\epsilon$ seems to increase at stronger mutation strengths, implying that local optima are more likely. It can also mean that the expected number of steps needed to escape local optima increases with mutation strength.

Within the upper bound of $\mu$ in the premise of Theorem \ref{theo:runtime_qap}, Algorithm \ref{alg:ea} seems to consume only a small portion of the budget using 2-opt, 3-opt and 4-opt, further reinforcing the observation that average case run-time is orders of magnitude smaller than worst-case. On the other hand, in case of $\lceil n/5\rceil$-opt, it exceeds the budgets near this bound at larger $n$. Since the mutation strengths are large in these instances, it suggests that $k$ values (in k-opt) contributes non-trivially to the expected run-time to optimum even when the lack of local optima is guaranteed, which agrees with the insight from Theorem \ref{theo:runtime_qap}.

As a side note, we observe from the experimental results on STSP in \cite{Do2020} that the algorithm, in unconstrained settings, fails to reach maximum diversity within the same budgets when the population size is close to a multiple of half the problem size (i.e. $\mu\approx kn/2$). This, when juxtaposed with the ``failure'' phenomena observed in our experiments, seems to mirror the difference in the upper bounds of the population size between Lemma \ref{lemma:ea_step_2opt_tsp} (STSP) and Lemma \ref{lemma:ea_step_kopt_qap} (QAP). While these are not demonstrated as comprehensively as in our experiment, we still find this similarity interesting, as it seems to suggest some universal property of mutation/local-search operators for different types of permutation-based representations (e.g. adjacency, assignment).

\subsection{Constrained diversity optimization}

In the constrained case, we look for the final diversity scores across varying thresholds $F=(1+\alpha)OPT$ and the extent to which optimizing for $D_2$ mitigate clustering, especially at small $\alpha$. Therefore, we consider $\alpha$ values 0.05, 0.2, 0.5, 1, and run the algorithm for $\mu n^2$ steps with no additional stopping criterion. Furthermore, we initiate the population with duplicates of the optimal solution to allow flexibility for meaningful behaviors. Aside from diversity scores, we also record the percentage of assignments belonging to exactly one solution (unique) out of $\mu n$ assignments in each final population.

\begin{table*}[ht]
\centering
\renewcommand{\arraystretch}{0.85}
\renewcommand\tabcolsep{2.1pt}
\caption{Diversity scores and the ratios of unique assignments in the final populations. The \hlc[gray!25]{\textbf{highlights}} denote greater values between the two approaches with statistical significance, based on Wilcoxon rank sum tests with 95\% confidence level.
}
\label{tb:qap_constrained}
\begin{footnotesize}
\begin{tabular}{crrcccccccccccc}
    \toprule
\multirow{3}{*}{\textbf{}}&\multirow{3}{*}{\textbf{$\mu$}} & \multirow{3}{*}{\textbf{$\alpha$}} &
\multicolumn{6}{c}{\bfseries Optimizing $D_1$} &\multicolumn{6}{c}{\bfseries Optimizing $D_2$}\\
\cmidrule(l{2pt}r{2pt}){4-9} \cmidrule(l{2pt}r{2pt}){10-15}
& & & \multicolumn{2}{c}{\bfseries $D_1$} & \multicolumn{2}{c}{\bfseries $D_2$} & \multicolumn{2}{c}{\bfseries unique percentage} & \multicolumn{2}{c}{\bfseries $D_1$} & \multicolumn{2}{c}{\bfseries $D_2$} & \multicolumn{2}{c}{\bfseries unique percentage} \\
\cmidrule(l{2pt}r{2pt}){4-5}\cmidrule(l{2pt}r{2pt}){6-7}\cmidrule(l{2pt}r{2pt}){8-9}\cmidrule(l{2pt}r{2pt}){10-11}\cmidrule(l{2pt}r{2pt}){12-13}\cmidrule(l{2pt}r{2pt}){14-15}
&&&mean&std&mean&std&mean&std&mean&std&mean&std&mean&std\\
\midrule
\multirow{20}{*}{\rotatebox[origin=c]{90}{Nug30}}&3&0.05&90.74\%&4.23\%&87.93\%&4.85\%&83.70\%&6.00\%&91.15\%&3.47\%&88.96\%&4.08\%&84.19\%&5.36\%\\&
&0.2&100.00\%&0.00\%&100.00\%&0.00\%&100.00\%&0.00\%&100.00\%&0.00\%&100.00\%&0.00\%&100.00\%&0.00\%\\&
&0.5&100.00\%&0.00\%&100.00\%&0.00\%&100.00\%&0.00\%&100.00\%&0.00\%&100.00\%&0.00\%&100.00\%&0.00\%\\&
&1&100.00\%&0.00\%&100.00\%&0.00\%&100.00\%&0.00\%&100.00\%&0.00\%&100.00\%&0.00\%&100.00\%&0.00\%\\\cmidrule(l{2pt}r{2pt}){2-15}
&10&0.05&\cellcolor[gray]{.9}\textbf{84.10\%}&2.05\%&48.06\%&10.09\%&23.00\%&4.13\%&81.71\%&1.87\%&\cellcolor[gray]{.9}\textbf{74.24\%}&2.47\%&\cellcolor[gray]{.9}\textbf{34.39\%}&2.44\%\\&
&0.2&100.00\%&0.00\%&100.00\%&0.00\%&100.00\%&0.00\%&100.00\%&0.00\%&100.00\%&0.00\%&100.00\%&0.00\%\\&
&0.5&100.00\%&0.00\%&100.00\%&0.00\%&100.00\%&0.00\%&100.00\%&0.00\%&100.00\%&0.00\%&100.00\%&0.00\%\\&
&1&100.00\%&0.00\%&100.00\%&0.00\%&100.00\%&0.00\%&100.00\%&0.00\%&100.00\%&0.00\%&100.00\%&0.00\%\\\cmidrule(l{2pt}r{2pt}){2-15}
&20&0.05&\cellcolor[gray]{.9}\textbf{84.00\%}&0.95\%&32.07\%&5.53\%&8.44\%&1.72\%&79.45\%&1.22\%&\cellcolor[gray]{.9}\textbf{68.61\%}&1.59\%&\cellcolor[gray]{.9}\textbf{16.72\%}&1.28\%\\&
&0.2&\cellcolor[gray]{.9}\textbf{99.95\%}&0.03\%&\cellcolor[gray]{.9}\textbf{99.09\%}&0.41\%&\cellcolor[gray]{.9}\textbf{98.98\%}&0.52\%&99.93\%&0.03\%&98.79\%&0.49\%&98.58\%&0.63\%\\&
&0.5&100.00\%&0.00\%&100.00\%&0.00\%&100.00\%&0.00\%&100.00\%&0.00\%&100.00\%&0.00\%&100.00\%&0.00\%\\&
&1&100.00\%&0.00\%&100.00\%&0.00\%&100.00\%&0.00\%&100.00\%&0.00\%&100.00\%&0.00\%&100.00\%&0.00\%\\\cmidrule(l{2pt}r{2pt}){2-15}
&50&0.05&\cellcolor[gray]{.9}\textbf{86.06\%}&0.71\%&17.44\%&2.50\%&2.58\%&0.49\%&79.98\%&0.81\%&\cellcolor[gray]{.9}\textbf{64.08\%}&1.01\%&\cellcolor[gray]{.9}\textbf{6.65\%}&0.52\%\\&
&0.2&\cellcolor[gray]{.9}\textbf{99.97\%}&0.01\%&90.62\%&0.48\%&19.79\%&0.26\%&99.72\%&0.02\%&\cellcolor[gray]{.9}\textbf{95.74\%}&0.24\%&\cellcolor[gray]{.9}\textbf{22.81\%}&0.41\%\\&
&0.5&\cellcolor[gray]{.9}\textbf{100.00\%}&0.00\%&91.28\%&0.39\%&20.00\%&0.00\%&100.00\%&0.00\%&\cellcolor[gray]{.9}\textbf{96.67\%}&0.00\%&\cellcolor[gray]{.9}\textbf{20.04\%}&0.04\%\\&
&1&\cellcolor[gray]{.9}\textbf{100.00\%}&0.00\%&91.41\%&0.48\%&20.00\%&0.00\%&100.00\%&0.00\%&\cellcolor[gray]{.9}\textbf{96.67\%}&0.00\%&\cellcolor[gray]{.9}\textbf{20.04\%}&0.06\%\\\hline
\multirow{20}{*}{\rotatebox[origin=c]{90}{Lipa90b}}&3&0.05&17.72\%&0.73\%&17.01\%&0.78\%&10.14\%&0.68\%&17.75\%&0.77\%&17.09\%&1.12\%&10.36\%&0.67\%\\&
&0.2&83.88\%&1.32\%&82.32\%&1.46\%&74.46\%&1.78\%&84.07\%&1.35\%&82.48\%&1.54\%&74.52\%&1.45\%\\&
&0.5&100.00\%&0.00\%&100.00\%&0.00\%&100.00\%&0.00\%&100.00\%&0.00\%&100.00\%&0.00\%&100.00\%&0.00\%\\&
&1&100.00\%&0.00\%&100.00\%&0.00\%&100.00\%&0.00\%&100.00\%&0.00\%&100.00\%&0.00\%&100.00\%&0.00\%\\\cmidrule(l{2pt}r{2pt}){2-15}
&10&0.05&17.44\%&0.40\%&13.85\%&1.30\%&9.11\%&0.38\%&17.48\%&0.43\%&\cellcolor[gray]{.9}\textbf{15.48\%}&0.59\%&\cellcolor[gray]{.9}\textbf{9.30\%}&0.22\%\\&
&0.2&78.23\%&1.33\%&44.26\%&11.52\%&38.70\%&8.79\%&\cellcolor[gray]{.9}\textbf{80.00\%}&0.62\%&\cellcolor[gray]{.9}\textbf{76.08\%}&0.87\%&\cellcolor[gray]{.9}\textbf{55.88\%}&0.55\%\\&
&0.5&100.00\%&0.00\%&100.00\%&0.00\%&100.00\%&0.00\%&100.00\%&0.00\%&100.00\%&0.00\%&100.00\%&0.00\%\\&
&1&100.00\%&0.00\%&100.00\%&0.00\%&100.00\%&0.00\%&100.00\%&0.00\%&100.00\%&0.00\%&100.00\%&0.00\%\\\cmidrule(l{2pt}r{2pt}){2-15}
&20&0.05&17.54\%&0.27\%&12.00\%&0.83\%&8.87\%&0.25\%&17.52\%&0.25\%&\cellcolor[gray]{.9}\textbf{14.92\%}&0.36\%&\cellcolor[gray]{.9}\textbf{9.26\%}&0.13\%\\&
&0.2&78.95\%&0.78\%&30.84\%&7.01\%&26.07\%&5.65\%&\cellcolor[gray]{.9}\textbf{79.62\%}&0.40\%&\cellcolor[gray]{.9}\textbf{74.70\%}&0.40\%&\cellcolor[gray]{.9}\textbf{54.94\%}&0.37\%\\&
&0.5&100.00\%&0.00\%&100.00\%&0.00\%&100.00\%&0.00\%&100.00\%&0.00\%&100.00\%&0.00\%&100.00\%&0.00\%\\&
&1&100.00\%&0.00\%&100.00\%&0.00\%&100.00\%&0.00\%&100.00\%&0.00\%&100.00\%&0.00\%&100.00\%&0.00\%\\\cmidrule(l{2pt}r{2pt}){2-15}
&50&0.05&\cellcolor[gray]{.9}\textbf{17.74\%}&0.21\%&9.38\%&0.52\%&7.98\%&0.46\%&17.60\%&0.17\%&\cellcolor[gray]{.9}\textbf{14.69\%}&0.24\%&\cellcolor[gray]{.9}\textbf{9.24\%}&0.10\%\\&
&0.2&\cellcolor[gray]{.9}\textbf{80.59\%}&0.45\%&15.16\%&3.04\%&10.44\%&2.21\%&78.71\%&0.27\%&\cellcolor[gray]{.9}\textbf{72.84\%}&0.27\%&\cellcolor[gray]{.9}\textbf{52.86\%}&0.29\%\\&
&0.5&100.00\%&0.00\%&100.00\%&0.00\%&100.00\%&0.00\%&100.00\%&0.00\%&100.00\%&0.00\%&100.00\%&0.00\%\\&
&1&100.00\%&0.00\%&100.00\%&0.00\%&100.00\%&0.00\%&100.00\%&0.00\%&100.00\%&0.00\%&100.00\%&0.00\%\\\hline
\multirow{20}{*}{\rotatebox[origin=c]{90}{Esc128}}&3&0.05&96.64\%&0.98\%&95.91\%&1.07\%&95.47\%&1.11\%&96.46\%&0.94\%&95.99\%&1.04\%&95.23\%&1.07\%\\&
&0.2&99.19\%&0.43\%&98.87\%&0.56\%&98.78\%&0.64\%&\cellcolor[gray]{.9}\textbf{99.41\%}&0.26\%&99.11\%&0.30\%&99.00\%&0.34\%\\&
&0.5&99.97\%&0.09\%&99.93\%&0.18\%&99.93\%&0.18\%&99.96\%&0.10\%&99.91\%&0.20\%&99.91\%&0.20\%\\&
&1&100.00\%&0.00\%&100.00\%&0.00\%&100.00\%&0.00\%&100.00\%&0.00\%&100.00\%&0.00\%&100.00\%&0.00\%\\\cmidrule(l{2pt}r{2pt}){2-15}
&10&0.05&95.69\%&1.17\%&85.97\%&4.53\%&83.07\%&4.06\%&95.51\%&0.49\%&\cellcolor[gray]{.9}\textbf{94.14\%}&0.48\%&\cellcolor[gray]{.9}\textbf{88.63\%}&0.69\%\\&
&0.2&98.80\%&0.70\%&94.35\%&4.32\%&91.72\%&4.18\%&98.93\%&0.22\%&\cellcolor[gray]{.9}\textbf{98.07\%}&0.36\%&\cellcolor[gray]{.9}\textbf{95.09\%}&0.55\%\\&
&0.5&99.92\%&0.06\%&99.48\%&0.27\%&99.24\%&0.52\%&\cellcolor[gray]{.9}\textbf{99.95\%}&0.04\%&\cellcolor[gray]{.9}\textbf{99.71\%}&0.18\%&\cellcolor[gray]{.9}\textbf{99.64\%}&0.27\%\\&
&1&100.00\%&0.00\%&100.00\%&0.00\%&100.00\%&0.00\%&100.00\%&0.00\%&100.00\%&0.00\%&100.00\%&0.00\%\\\cmidrule(l{2pt}r{2pt}){2-15}
&20&0.05&\cellcolor[gray]{.9}\textbf{96.39\%}&0.89\%&81.10\%&6.59\%&76.90\%&7.14\%&94.83\%&0.34\%&\cellcolor[gray]{.9}\textbf{93.08\%}&0.29\%&\cellcolor[gray]{.9}\textbf{84.70\%}&0.41\%\\&
&0.2&\cellcolor[gray]{.9}\textbf{99.06\%}&0.17\%&93.57\%&2.22\%&89.11\%&1.74\%&98.68\%&0.16\%&\cellcolor[gray]{.9}\textbf{97.44\%}&0.19\%&\cellcolor[gray]{.9}\textbf{91.23\%}&0.45\%\\&
&0.5&\cellcolor[gray]{.9}\textbf{99.90\%}&0.05\%&99.06\%&0.46\%&98.17\%&1.02\%&99.86\%&0.05\%&99.29\%&0.08\%&97.84\%&0.82\%\\&
&1&100.00\%&0.00\%&100.00\%&0.00\%&100.00\%&0.00\%&100.00\%&0.00\%&100.00\%&0.00\%&100.00\%&0.00\%\\\cmidrule(l{2pt}r{2pt}){2-15}
&50&0.05&\cellcolor[gray]{.9}\textbf{96.81\%}&1.04\%&65.25\%&17.04\%&57.76\%&17.24\%&94.52\%&0.23\%&\cellcolor[gray]{.9}\textbf{92.00\%}&0.21\%&\cellcolor[gray]{.9}\textbf{82.12\%}&0.18\%\\&
&0.2&\cellcolor[gray]{.9}\textbf{98.99\%}&0.13\%&88.62\%&3.71\%&83.69\%&4.21\%&98.38\%&0.10\%&\cellcolor[gray]{.9}\textbf{96.42\%}&0.16\%&\cellcolor[gray]{.9}\textbf{87.93\%}&0.34\%\\&
&0.5&\cellcolor[gray]{.9}\textbf{99.92\%}&0.02\%&98.01\%&0.59\%&96.30\%&1.05\%&99.76\%&0.05\%&\cellcolor[gray]{.9}\textbf{98.79\%}&0.09\%&\cellcolor[gray]{.9}\textbf{95.31\%}&0.37\%\\&
&1&100.00\%&0.00\%&100.00\%&0.01\%&100.00\%&0.01\%&100.00\%&0.00\%&99.99\%&0.02\%&99.99\%&0.02\%\\\hline
\end{tabular}
\end{footnotesize}
\end{table*}

Table \ref{tb:qap_constrained} shows a comparison in terms of $D_1$ and $D_2$ score averages as well as unique assignment percentage averages. Overall, maximum diversity is achieved reliably in most cases when $\alpha\geq0.5$. For Lipa90b, there are tremendous gaps in final diversity scores when $\alpha$ changes from $0.05$ to $0.2$, from $61\%$ to $66\%$. The differences are much smaller in other QAPLIB instances, no more than $16\%$. Also, at $\alpha=0.5$, maximum diversity is not reached as frequently for Esc128 as for other instances. These suggest significantly different cost distributions in the solution spaces associated with these QAPLIB instances. More specifically, it seems that the connected (via 2-opt neighborhood) region of 1.05-approximation solutions around the optimum is much smaller in Lipa90b than it is in other instances. Additionally, unique assignment percentages follow the same trend, and understandably decrease with bigger populations as uniqueness diminishes. This drop is most severe in Nug30 cases, since the increase in $\mu$ is the largest relative to the instance size.

Comparing the diversity scores from the two approaches, we can see trends consistent with those in the unconstrained case. Each approach predictably excels at maximizing its own measure over the other. That said, the $D_2$ approach does not fall far behind in $D_1$ scores, even in cases where statistical significance is observed (at most $7\%$ difference). Meanwhile, the $D_1$ approach's $D_2$ scores are much lower than those of the other, especially in hard cases (small $\alpha$ and large $\mu$), up to $46\%$ gap. This indicates that using the measure $\mathcal{D}$, Algorithm \ref{alg:ea} significantly reduces clustering, and equalizes assignments' representations almost as effectively as when using the measure $\mathcal{N}$. Additionally, we can observe similar differences in the percentages of unique assignments, which seem to correlate with $D_2$ stronger than with $D_1$.

Based on the standard deviations, it seems that the outputs are mostly stable. This suggests that the algorithm converges within the budget in most cases, and that it reaches similarly diverse populations across runs. The exceptions are when $\mathcal{N}$ is used at $\alpha\leq0.2$, the $D_2$ scores achieved exhibit relatively large variances. This agrees with the observation made in Figure \ref{fig:qap_unconstrained}, and confirms the existence of satisfactory populations of similar $D_1$ diversity that have wildly different $D_2$.

\section{Conclusion}\label{sec:conclusion}

We studied evolutionary diversity optimization in the Traveling Salesperson Problem and Quadratic Assignment Problem. In this type of optimization problem, the goal is to maximize diversity as quantified by some metric, and the constraint involves the solutions' qualities. We described the similarity and difference between the structure of a STSP/ATSP tour and that of a QAP mapping, and customized two diversity measures to each problem. We considered a baseline $(\mu+1)$ evolutionary algorithm that incrementally modifies the population using traditional mutation operators on one solution at a time, and scenarios where solutions are accepted regardless of quality. We showed that for any sufficiently small $\mu$, the algorithm guarantees maximum diversity in STSP within using 2-opt and 4-opt within $\mathcal{O}(\mu^2n^3)$ expected iterations, while 3-opt suffers from local optima even with very small $\mu$. For ATSP, we proved a worst-case run-time of $\mathcal{O}(\mu^2n^4)$ from using 3-opt and 4-opt, under more generous upper bounds of $\mu$. Lastly, we showed that in QAP, the algorithm reaches maximum diversity in $\mathcal{O}(2^{k-2}\mu^2n^2(n-k+1))$ steps using k-opt mutation as we describe in this work. In the same proof, we showed that the upper bound of $\mu$ guaranteeing the lack of local optima decreases as the mutation strength increases.

Additional experiments on QAPLIB instances shed light on differences on evolutionary trajectories when optimizing for the two diversity measures. Our results show heterogeneity in the correlation between the quality constraint threshold and the achieved diversity across different instances, and that the average practical performance is much more optimistic than the worst-case suggests. Furthermore, our experiment in unconstrained scenarios indicated that the algorithm tends to encounter local optima when $\mu$ is close to multiples of $n$, and that these detrimental regions expand with increasing mutation strength, agreeing with our theoretical insight.

In realistic settings, obtaining a solution satisfying a given quality criterion can be very difficult, an aspect that is not rigorously addressed in this work. For NP-hard problems like TSP and QAP, this typically necessitates exponential-time algorithms. For any-time algorithms, one might consider multi-modal optimization approaches to evolve many good solutions simultaneously. It is known that the number of solutions affect multi-modal searches' convergence speed tremendously. Therefore, it is an interesting challenge to craft efficient niching-based algorithms, for instance, to tackle the diverse solutions problem; one can expect that the desired population to this problem deviates significantly from that to the multi-modal problem on the same fitness landscape.

\section*{Acknowledgments}
This work was supported by the Phoenix HPC service at the University of Adelaide, and by the Australian Research Council through grant DP190103894.

\bibliographystyle{unsrt}
\bibliography{refs}

\end{document}